%% file: neurips_2024.tex
\documentclass{article}

    \PassOptionsToPackage{numbers, compress}{natbib}

\usepackage[final]{neurips_2024}

\newcommand{\shortname}{CONTRAST}

\usepackage[utf8]{inputenc} 
\usepackage[T1]{fontenc}    
\usepackage[hidelinks]{hyperref}       
\usepackage{url}            
\usepackage{booktabs}       
\usepackage{amsfonts}       
\usepackage{nicefrac}       
\usepackage{microtype}      
\usepackage{tabularx} 
\usepackage[dvipsnames]{xcolor} 
\definecolor{citecolor}{HTML}{0071bc}
\hypersetup{
    colorlinks,
    linkcolor={citecolor},
    citecolor={citecolor},
    urlcolor={citecolor}
}
\usepackage{graphicx}
\usepackage{subfigure}
\usepackage{times}
\usepackage{epsfig}
\usepackage{amsmath}
\usepackage{amssymb}
\usepackage{optidef}
\usepackage{float}
\usepackage{titletoc}
\usepackage{algorithmic} 
\usepackage{colortbl}
\usepackage{multicol}
\usepackage{booktabs}
\usepackage{multirow}
\usepackage{diagbox}
\usepackage{enumitem}
\usepackage{soul}
\usepackage{array}
\usepackage{color, colortbl}
\usepackage{subcaption}
\usepackage{adjustbox}
\usepackage{pifont}
\usepackage{amsthm} 
\usepackage{wrapfig}
\usepackage{comment}

\let\oldlistofalgorithms\listofalgorithms
\let\listofalgorithms\relax

\usepackage[ruled,vlined]{algorithm2e}

\let\listofalgorithms\oldlistofalgorithms

\newtheorem{theorem}{Theorem}

\newcommand{\cmark}{\ding{51}}%
\newcommand{\xmark}{\ding{55}}%

\newcommand{\w}{\mathrm{w}}
\newcommand{\f}{\mathrm{f}}
\newcommand{\Lc}{{\cal{L}}}
\newcommand{\lopt}{{\cal{L}}^{\star(t)}_T}

\title{\shortname{}: Continual Multi-source Adaptation to Dynamic Distributions}

\author{%
  Sk Miraj Ahmed$^{2,}$\thanks{Equal contribution; Co-first authors listed alphabetically by last name.} \thanks{Currently at Brookhaven National Laboratory. Work done while the author was at UCR.} \hspace{0.1cm},  Fahim Faisal Niloy$^{1,}$\footnotemark[1] \hspace{0.1cm}, Xiangyu Chang$^{1}$, Dripta S. Raychaudhuri$^{3,}$\thanks{Currently at AWS AI Labs. Work done while the author was at UCR.} \hspace{0.1cm}, \\ \textbf{Samet Oymak}$^{4}$, \hspace{0.1cm}  \textbf{Amit K. Roy-Chowdhury}$^{1}$\\
  $^{1}$University of California, Riverside, $^{2}$Brookhaven National Laboratory, $^{3}$AWS AI Labs, \\ $^{4}$University of Michigan, Ann Arbor\\
  \texttt \small \{sahme047@, fnilo001@,
  cxian008@,
  drayc001@, amitrc@ece.\}ucr.edu, oymak@umich.edu
}

\begin{document}

\maketitle

\begin{abstract}
Adapting to dynamic data distributions is a practical yet challenging task. One effective strategy is to use a model ensemble, which leverages the diverse expertise of different models to transfer knowledge to evolving data distributions. However, this approach faces difficulties when the dynamic test distribution is available only in small batches and without access to the original source data. To address the challenge of adapting to dynamic distributions in such practical settings, we propose \underline{CON}tinual mul\underline{T}i-sou\underline{R}ce \underline{A}daptation to dynamic
di\underline{S}tribu\underline{T}ions (\shortname), 
a novel method that optimally combines multiple source models to adapt to the dynamic test data. 
\shortname~has two distinguishing features. First, it efficiently computes the optimal \emph{combination} weights to combine the source models to adapt to the test data distribution continuously as a function of time. Second, it identifies which of the source model parameters to update so that only the model which is most correlated to the target data is adapted, 
leaving the less correlated ones untouched; this mitigates the issue of ``forgetting" the source model parameters by focusing only on the source model that exhibits the strongest correlation with the test batch distribution. Through theoretical analysis we show that the proposed method is able to optimally combine the source models and prioritize updates to the model least prone to forgetting. Experimental analysis on diverse datasets demonstrates that the combination of multiple source models does at least as well as the best source (with hindsight knowledge), and performance does not degrade as the test data distribution changes over time (robust to forgetting).
\end{abstract}

\section{Introduction}
\label{submission}

Deep neural networks have shown impressive performance on test inputs that closely resemble the training distribution. However, their performance degrades significantly when they encounter test inputs from a different data distribution. Unsupervised domain adaptation (UDA) techniques~\cite{tzeng2017adversarial,tsai2018learning} aim to mitigate this performance drop. Addressing the distribution shift in case of \emph{dynamic data distributions} is even more challenging and practically relevant - in many real-world applications like autonomous navigation, models often encounter dynamically evolving distributions. Furthermore, test data is often accessed in streaming batches rather than all at once, and source data may not always be available due to privacy and storage concerns.



For domain adaptation to dynamically evolving environments, employing a model ensemble can be beneficial, as it allows leveraging the learned knowledge of different models to more effectively mitigate dynamic distribution shifts. Additionally, situations may arise wherein the user has access to a diverse set of pre-trained models across distinct source domains, and no access to source domain data corresponding to each model due to privacy, storage or other constraints. Consequently, training a unified model using the combined source data becomes unfeasible. In those scenarios, it is both reasonable and effective to employ and adapt the entire available array of source models during testing, thereby enhancing performance beyond the scope of single source model adaptation. Moreover, employing a model ensemble provides the flexibility to effortlessly incorporate or exclude models post-deployment, aligning with the user's preferences and the needs of the given task. This flexibility is not achievable with a single domain-generalized model trained on combined source data.

\begin{figure}[t!]
\centering
\includegraphics[width=0.8\columnwidth] {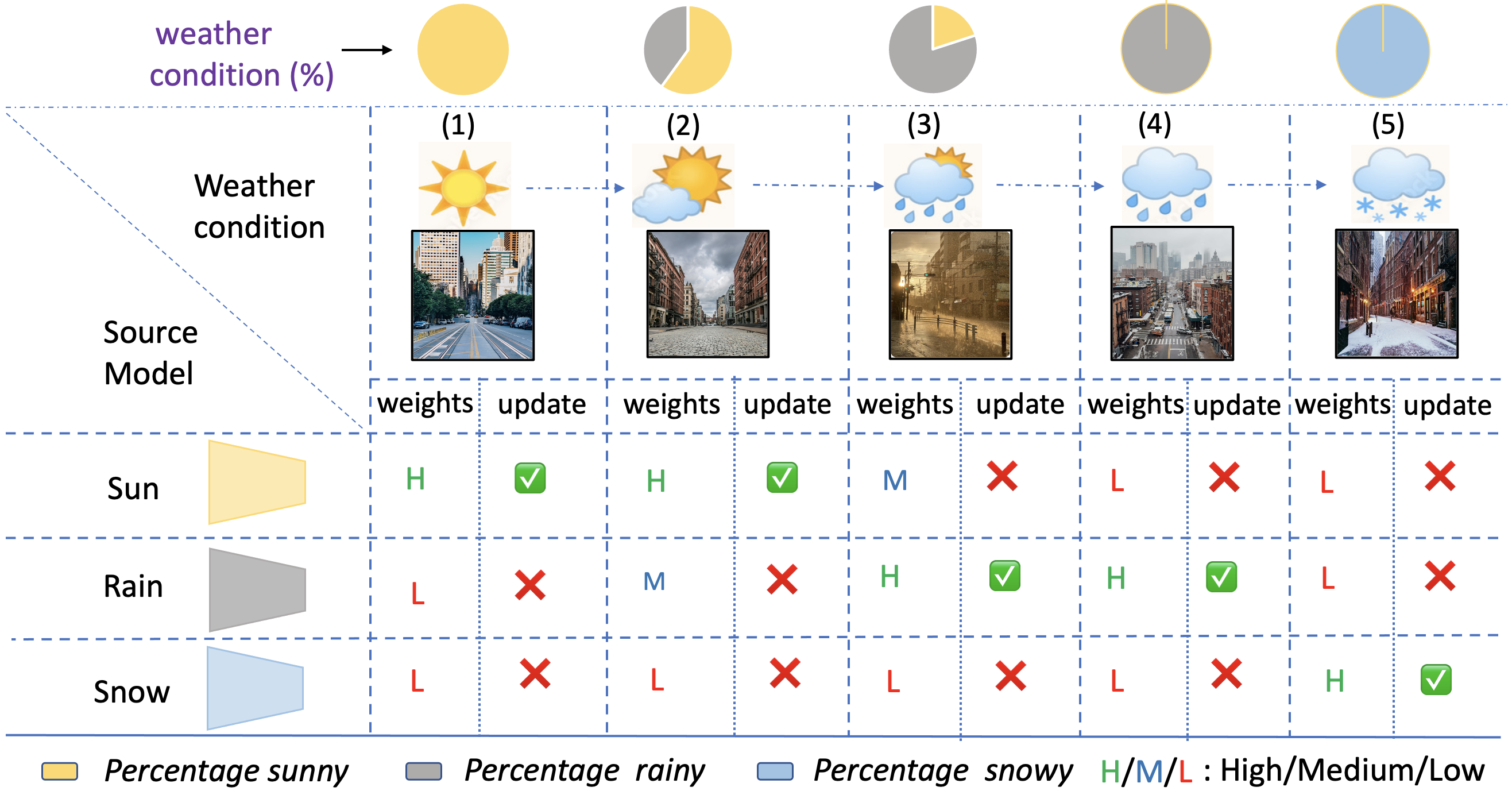}
\caption{\textbf{Problem setup.}  Consider several source models trained using data from different weather conditions. During the deployment of these models, they may encounter varying weather conditions that could be a combination of multiple conditions in varying proportions (represented by the pie charts on top). Our goal is to infer on the test data using the ensemble of models by automatically figuring out proper combination weights and adapting the appropriate models on the fly.} 
\label{fig:intro}
\vspace{-1em}
\end{figure}



As an example, consider a scenario where a recognition model, initially trained on clear weather conditions, faces data from mixed weather scenarios, like sunshine interspersed with rain (see Figure~\ref{fig:intro}). In such cases, employing multiple models - specifically those trained on clear weather and rain — with appropriate weighting can potentially reduce the test error as opposed to relying on a single source model. In this context, the models for clear weather and rain would be assigned higher weights, while models for other weather conditions would receive relatively lesser weightage.

The main challenge of developing such a model ensembling method is to \emph{learn appropriate combination weights to optimally combine the source model ensemble during the test phase as data is streaming in, such that it results in a test error equal or lower than that of the best source model}. 
To solve this, we propose \underline{CON}tinual mul\underline{T}i-sou\underline{R}ce \underline{A}daptation to dynamic
di\underline{S}tribu\underline{T}ions (\shortname) that handles multiple source models and optimally combines them to adapt to the test data.


The efficacy of using multiple source models also extends to preventing \emph{catastrophic forgetting} that may arise when adapting to dynamic distributions for a prolonged time. Consider again the scenario of multiple source models, each trained on a different weather condition. During inference, only the parameters of the models most closely related to the weather encountered during test time will get updates, and the unrelated ones will be left untouched. This ensures that the model parameters do not drift too far from the initial state, since only those related to the test data are being updated. This mechanism mitigates forgetting when the test data distribution varies over a long time scale, as is likely to happen in most realistic conditions. Even if an entirely unrelated distribution appears during testing and there is no one source model to handle it, the presence of multiple sources can significantly reduce the rate at which the forgetting  occurs. This is again because only the most closely related models (clear and rainy weather in the example above) are updated, while others (e.g., snow) are left untouched. Our setting is closely related to Test Time Adaptation methods (TTA)~\cite{wang2020tent}, and ours is the first to address \emph{adaptation of multiple sources for dynamically shifting distributions} during test time.


\noindent\textbf{Main Contributions.} 
Our proposed approach, \shortname, makes the following contributions.
\begin{itemize}[leftmargin=*,topsep=0pt]
\setlength\itemsep{-1pt}
    \item We propose a framework for multi-source adaptation to dynamic distribution shifts from streaming test data and without access to the source data. 
    Our approach has the ability to merge the source models using appropriate combination weights during test time, enabling it to perform just as well as the best-performing source or even surpass it.
    \item Our framework achieves performance on par with the best-performing source and also effectively mitigates catastrophic forgetting when faced with long-term, fluctuating test distributions.
    
    
    \item We provide theoretical insights on \shortname, illustrating how it addresses domain shift by optimally combining source models and prioritizing updates to the model least prone to forgetting.
    \item To demonstrate the real-world advantages of our methodology, we perform experiments on a diverse range of benchmark datasets. 
\end{itemize}

\section{Related Works} 
\noindent \textbf{Unsupervised Domain Adaptation.} 
UDA methods have been applied to many machine learning tasks, including image classification~\cite{tzeng2017adversarial}, semantic segmentation~\cite{tsai2018learning}, object detection~\cite{hsu2020progressive} and reinforcement learning~\cite{raychaudhuri2021cross}, in an effort to address the data distribution shift. Most approaches try to align the source and target data distributions, using techniques such as maximum mean discrepancy~\cite{long2015learning}, adversarial learning~\cite{ganin2016domain,tzeng2017adversarial,raychaudhuri2021cross} and image translation ~\cite{Hoffman_ICML_2018,Luan_CVPR_2019}.
Recently, there has been a growing interest in adaptation using only a pre-trained source model due to privacy and memory storage concerns related to the source data \cite{yang2021generalized,kundu2020universal,wang2022exploring, yang2020unsupervised, li2020model, ding2022source, chen2022contrastive}. These approaches include techniques such as information maximization~\cite{liang2020we,ahmed2021unsupervised,ahmed2022cross}, pseudo labeling~\cite{yeh2021sofa,kumar2023conmix}, and self-supervision~\cite{xia2021adaptive}. 
\input{tables/problem_setting}

\noindent \textbf{Multi-Source Domain Adaptation (MSDA).} 
Both UDA and source-free UDA have been extended to multi-source setting by incorporating knowledge from multiple source models \cite{ahmed2021unsupervised, zhao2020multi}. Notable techniques include discrepancy-based MSDA \cite{guo2018multi}, higher-order moments \cite{peng2019moment}, adversarial methods \cite{xu2018deep}, and Wasserstein distance-based methods \cite{li2018extracting}. However, these methods are specifically tailored to UDA scenarios, where the whole target data is assumed to be available during adaptation. Whereas, in our setting we consider access to a batch of target data at an instance. Another related field is Domain Generalization (DG) \cite{dubey2021adaptive, xiao2022learning} , which refers to training a single model on a combined set of data from different source domains. Hence, DG requires data from all distinct domains to be available altogether during training, which may not be always feasible. 
Additionally, Model Soups \cite{wortsman2022model} is a popular approach to ensemble models fine-tuned on same data distribution, where the weights of multiple models are averaged for inference. On the other hand, we use a weighting approach for model predictions, where models are pre-trained on different source data distributions. In our  problem, inspired by MSDA, \textit{users are only provided with pre-trained source models}.

\vskip 2pt
\noindent \textbf{Adaptation to Dynamic Data.} Few works \cite{Wu_2019_ICCV, rostami2021lifelong, panagiotakopoulos2022online} have addressed the adaptation to dynamic data distributions. However, these works either require source data or the entire target domain data to be available during adaptation. When additional constraints such as streaming target data batches and no access to source data are considered, the setting closely aligns with Test Time Adaptation (TTA). While UDA methods typically require a substantial volume of target domain data for model adaptation, which is performed offline and prior to deployment, TTA adjusts a model post-deployment, during inference or testing.
One of the early works \cite{li2016revisiting} use test-batch statistics for batch normalization adaptation. Tent \cite{wang2020tent} updates a pre-trained source model by minimizing entropy and updating batch-norm parameters. DUA \cite{mirza2022norm} updates batch-norm stats with incoming test batches. TTA methods have also been applied to segmentation problems \cite{valanarasu2022fly, shin2022mm, liu2021source, hu2021fully}. When these TTA methods are used to adapt to changing target distribution, they usually suffer from `forgetting' and `error accumulation' \cite{wang2022continual}. In order to solve this, CoTTA \cite{wang2022continual} restores source knowledge stochastically to avoid drifting of source knowledge. EATA \cite{niu2022efficient} adds a regularization loss to preserve important weights for less forgetting. While motivated by TTA, our method considers multi-source adaptation in a dynamic setting and has an inherent capability to mitigate forgetting. In Table \ref{tab:prob_setting}, we illustrate a comparison between our setting and existing settings.

\section{\shortname \ Framework}
\subsection{Problem Setting}

In this problem setting, we propose to combine multiple pre-trained models during test time through the application of suitable combination weights, determined based on a limited number of test samples. Specifically, we will focus on the classification task that involves $K$ categories. Consider the scenario where we have a collection of $N$ source models, denoted as $\{\mathrm{f}^j_S\}_{j=1}^N$, that we aim to deploy during test time. In this situation, we assume that a sequence of test data  $\{x_i^{(1)}\}_{i=1}^B\rightarrow \{x_i^{(2)}\}_{i=1}^B\rightarrow \ldots \{x_i^{(t)}\}_{i=1}^B \rightarrow \ldots$ are coming batch by batch in an online fashion, where $t$ is the index of time-stamp and $B$ is the number of samples in the test batch. We also denote the test distribution at time-stamp $t$ as $\mathcal{D}_T^{(t)}$, which implies $\{x_i^{(t)}\}_{i=1}^B\sim\mathcal{D}_T^{(t)}$. Motivated by \cite{ahmed2021unsupervised}, we model the test distribution in each time-stamp $t$ as a linear combination of source distributions where the combination weights are denoted by $\{\mathrm{w}^{(t)}_j\}_{j=1}^N$. Thus, our inference model on test batch $t$ can be written as  $\mathrm{f}_T^{(t)} = \sum_{j=1}^N \mathrm{w}_j^{(t)} \mathrm{f}^{j(t)}_S$ where $\mathrm{f}^{j(t)}_S$ is the adapted $j$-th source in time stamp $t$. Based on this setup our objective is twofold:

\begin{enumerate}[left=0pt]
\item We want to determine the optimal combination weights $\{\mathrm{w}^{(t)}_j\}_{j=1}^N$ for the current test batch such that the test error for the optimal inference model is lesser than or equal to the test error of best source model. Mathematically we can write this as follows:
    \begin{equation}
        \epsilon^{(t)}_{test}(\mathrm{f}_T^{(t)}) \leq \underset{1\leq j\leq N}{\text{min}} \ \epsilon^{(t)}_{test}(\mathrm{f}_S^j),
        \label{obj1}
    \end{equation}
    where $\epsilon^{(t)}_{test}(\cdot)$ evaluates the test error on $t$-th batch.

    \item We also aim for the model to maintain consistent performance on source domains, as it progressively adapts to the changing test conditions. This is necessary to ensure that the model has not catastrophically forgotten the original training distribution of the source domain and maintains its original performance if the source data is re-encountered in the future
    We would ideally want to have:
    \begin{equation}
        \epsilon_{src}(\mathrm{f}^{j(t)}_S) \approx \epsilon_{src}(\mathrm{f}_S^j) \quad \forall j,t,
        \label{obj2}
    \end{equation}
    where, $\epsilon_{src}(\mathrm{f}_S^j)$ denote the test error of $j$-th source on its corresponding test data when using the original source model $\mathrm{f}_S^j$, whereas $\epsilon_{src}(\mathrm{f}^{j(t)}_S)$ represents the test error on the same test data using the $j$-th source model adapted up to time step $t$, denoted as $\mathrm{f}^{j(t)}_S$.
    
\end{enumerate}

\subsection{Overall Framework}

Our framework undertakes two operations on each test batch. First, we learn the combination weights for the current batch at time step $t$ by freezing the model parameters. Then, we update the model corresponding to the largest weight with existing state-of-the-art TTA methods, which allows us to fine-tune the model and improve its performance. This implies that the model parameters of source $j$ might get updated up to $p$ times at time-step $t$, where $0 \leq p \leq (t-1)$.

In other words, the states of the source models evolve over time depending on the characteristics of the test batches up to the previous time step. To formalize this concept, we define the state of the source model $j$ at time-step $t$ as $\mathrm{f}_S^{j(t)}$.
In the next section, we will provide a detailed explanation of both aspects of our framework: (i) learning the combination weights, and (ii) updating the model parameters. By doing so, we aim to provide a comprehensive understanding of how our approach works in practice.

\begin{figure*}[h]
\begin{center}
    \includegraphics[width= 1\textwidth]{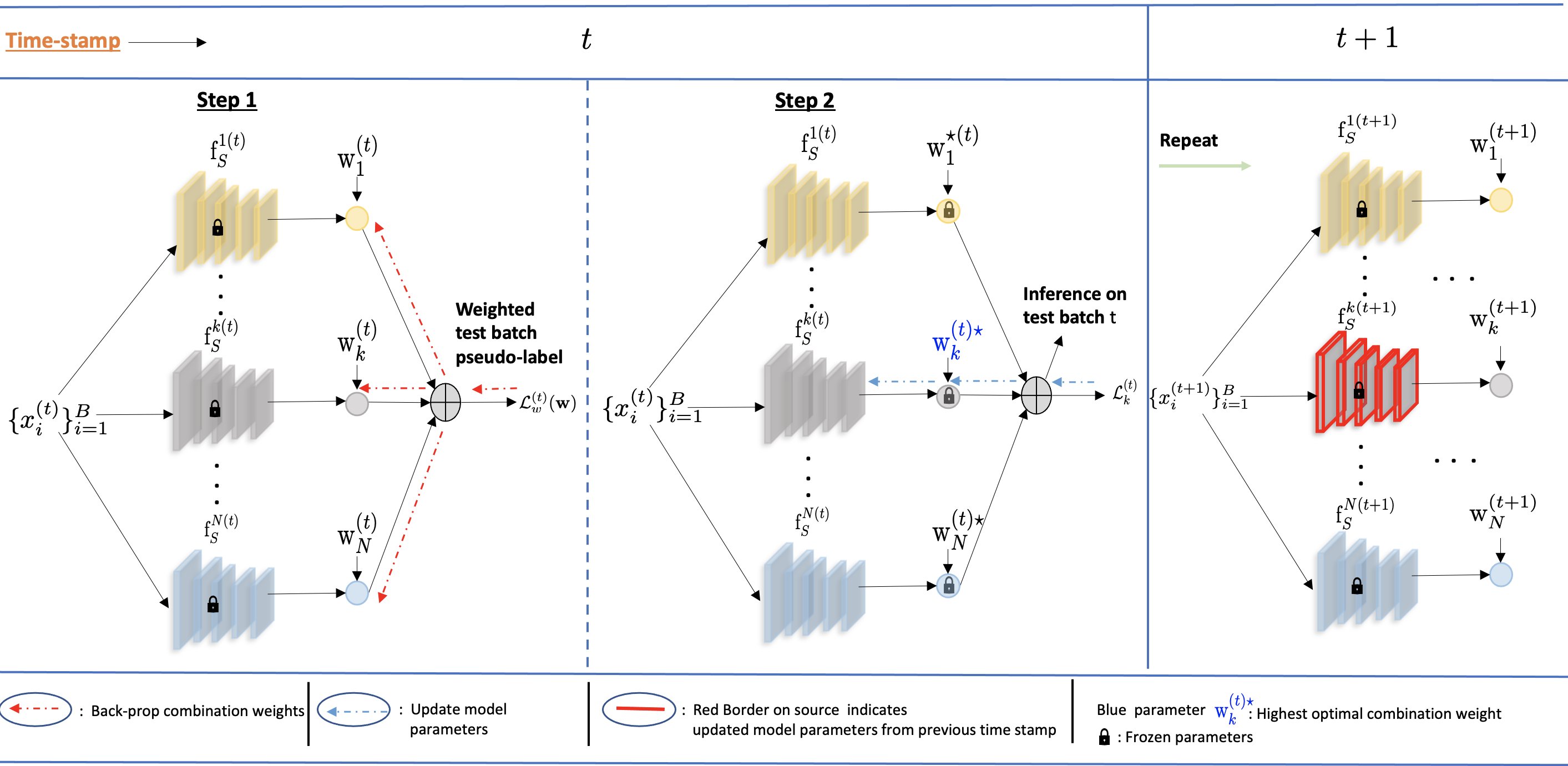}
\end{center}
\caption{\textbf{Overall Framework.} During test time, we aim to adapt multiple source models in a manner such that it optimally blends the sources with suitable weights based on the current test distribution. Additionally, we update the parameters of only one model that exhibits the strongest correlation with the test distribution.} 
\label{fig:framework}
\vspace{-1em}
\end{figure*}

\subsection{Learning the combination weights}
\label{sec:learning_comb_weight}
For an unlabeled target sample $x_i^{(t)}$ that arrives at time-stamp $t$, we denote its pseudo-label, as predicted by source $j$, as $\hat{y}_{ij}^{(t)}=\mathrm{f}_S^{j(t)}(x_i^{(t)})$, where $\mathrm{f}_S^{j(t)}$ is the state of source $j$ at time-stamp $t$. 
Now we linearly combine these pseudo-labels by source combination weights 
$\mathbf{w}=[\mathrm{w}_1 \ \mathrm{w}_2 \ldots \mathrm{w}_N]^\top \in \mathbb{R}^N$ to get weighted pseudo-label $\hat{y}_i^{(t)} = \sum_{j=1}^N \mathrm{w}_j \hat{y}_{ij}^{(t)}$. Using these weighted pseudo-labels for all the samples in the $t$-th batch we calculate the expected Shannon entropy as,
\begin{equation}
    \mathcal{L}_{w}^{(t)}(\mathbf{w}) = -\mathbb{E}_{\mathcal{D}_T^{(t)}} \sum_{c=1}^K \hat{y}_{ic}^{(t)} \log(\hat{y}_{ic}^{(t)}) 
    \label{w_ent}
\end{equation}
Based on this loss we solve the following optimization:
\begin{mini}|l|
{\mathbf{w}}{\mathcal{L}_w^{(t)}(\mathbf{w})}{}{}
\addConstraint {\mathrm{w}_j \geq 0, \forall j \in \{1, 2, \dots, N\}} 
\addConstraint {\sum_{j=1}^n \mathrm{w}_j=1}
\label{opt:main_opt}
\end{mini}
%
%
Suppose we get $\mathbf{w}^{\star(t)}$ to be the optimal combination weight vector by performing the optimization in (\ref{opt:main_opt}). In such case, the optimal inference model for test batch $t$ can be expressed as follows:
\begin{equation}
    \mathrm{f}_T^{(t)} = \sum_{j=1}^N \mathrm{w}_j^{\star(t)} \mathrm{f}^{j(t)}_S
    \label{model_inf}
\end{equation}
Thus, by learning ${\mathbf{w}}$ in this step, we satisfy Eqn.~\eqref{obj1}.

\noindent\textbf{Model parameter update.} 
\label{param_update}
After obtaining $\mathbf{w}^{\star(t)}$, next we select the most relevant source model $k$ given by $k= \text{arg} \ \underset{1\leq j\leq N}{\text{max}} \ \mathrm{w}^{\star(t)}_j$. This indicates that the distribution of the current test batch is most correlated with the source model $k$. We then adapt model $k$ to the test batch $t$ using any state-of-the-art single source method that adapts to dynamic target distributions. Specifically, we employ three distinct adaptation approaches: (i) TENT \cite{wang2020tent}, (ii) CoTTA \cite{wang2022continual}, and (iii) EaTA \cite{niu2022efficient}. 

\noindent\textbf{Optimization strategy for (\ref{opt:main_opt}).} 
Solving the optimization problem in Eq.~\ref{opt:main_opt} is a prerequisite for inferring the current test batch. As inference speed is critical for test-time adaptation, it is desirable to learn the weights quickly. To achieve this, we design two strategies: (i) selecting an appropriate initialization for $\mathbf{w}$, and (ii) determining an optimal learning rate.\\
\noindent\textbf{(i) Initialization:} Pre-trained models contain information about expected batch mean and variance in their Batch Norm (BN) layers based on the data they were trained on. To leverage this information, we extract these stored values from each source model prior to adaptation. Specifically, we denote the expected batch mean and standard deviation for the $l$-th layer of the $j$-th source model as $\mu_l^j$ and $\sigma_l^j$, respectively.

During testing on the current batch $t$, we pass the data through each model and extract its mean and standard deviation from each BN layer. We denote these values as $\mu_l^{T(t)}$ and $\sigma_l^{T(t)}$, respectively. One useful metric for evaluating the degree of alignment between the test data and each source is the distance between their respective batch statistics. A smaller distance implies a stronger correlation between the test data and the corresponding source. Assuming that the batch-mean statistic per node of the BN layers to be a univariate Gaussian, we calculate the distance (KL divergence) between the $j$-th source (approximated as ${\cal{N}}(\mu_l^j,(\sigma_l^{j})^2)$) and the $t$-th test batch  (approximated as ${\cal{N}}(\mu_l^{T(t)},(\sigma_l^{T(t)})^2)$) as follows (derivation in Appendix Section \ref{sec:KL}):
\begin{equation}
\small
\begin{split}
    & \mathrm{\theta}_j^t  = \sum_{l}\mathcal{D}_{KL} \left [ \mathcal{N}\left(\mu_l^{T(t)},(\sigma_l^{T(t)})^2\right), \mathcal{N}\left(\mu_l^j,(\sigma_l^j)^2\right)\right] 
                    = \\ & \sum_{l=1}^{n_j} \sum_{m=1}^{d_j^l} \log \left (\frac{\sigma_{lm}^j }{\sigma_{lm}^{T(t)}} \right)
                   {+} \frac{\left(\sigma_{lm}^{T(t)}\right)^2{+}\left(\mu_{lm}^j-\mu_{lm}^{T(t)}\right)^2}{2 \left(\sigma_{lm}^{j}\right)^2}  {-} \frac{1}{2} \nonumber
\end{split}
\end{equation}
where subscript $lm$ denotes the $m$-th node of $l$-th layer. After obtaining the distances, we use a softmax function denoted by $\delta(\cdot)$ to normalize their negative values. The softmax function is defined as $\delta_j(a)=\frac{\exp(a_j)}{\sum_{i=1}^N \exp(a_i)}$, where $a \in \mathbb{R}^N$, and $j\in {1,2,\dots,N}$. If $\mathbf{\theta}^t = \left[\mathrm{\theta}_1^t, \mathrm{\theta}_2^t \ldots \mathrm{\theta}_N^t\right]^\top \in \mathbb{R}^N$ is the vectorized form of the distances from all the sources, we set
\begin{equation}
  \mathbf{w}_{init}^{(t)}= \delta(-\mathbf{\theta}^t)
  \label{w_init}
\end{equation}
where $\mathbf{w}_{init}^{(t)}$ is the initialization for $\mathbf{w}$. As we shall see, this choice leads to a substantial performance boost compared to random initialization.

\noindent\textbf{(ii) Optimal step size:} Since we would like to ensure rapid convergence of optimization in Eqn.~\ref{opt:main_opt}
, we select the optimal step size for gradient descent in the initial stage. Given an initialization $\mathbf{w}_{init}^{(t)}$ and a step size $\alpha^{(t)}$, we compute the second-order Taylor series approximation of the function $\mathcal{L}_w^{(t)}$ at the updated point after one gradient step. Next, we determine the best step size $\alpha^{(t)}_{best}$ by minimizing the approximation with respect to $\alpha^{(t)}$. This is essentially an approximate Newton's method (details in Appendix section \ref{sec:newton}) and has a closed-form solution given by
\begin{equation}
\scriptsize
    \alpha_{best}^{(t)} = \left[\left(\nabla_{\mathbf{w}}\mathcal{L}_w^{(t)}\right)^\top \left(\nabla_{\mathbf{w}}\mathcal{L}_w^{(t)}\right)/\left(\nabla_{\mathbf{w}}\mathcal{L}_w^{(t)}\right)^\top \mathcal{H}_w \left(\nabla_{\mathbf{w}} \mathcal{L}_w^{(t)}\right)\right] \Bigg|_{\mathbf{w}^{init}}.
    \label{alpha}
\end{equation}
Here $\nabla_{\mathbf{w}}\mathcal{L}_w^{(t)}$ and $\mathcal{H}_w$ are the gradient and Hessian of $\mathcal{L}_w^{(t)}$ with respect to $\mathbf{w}$. Together with $\mathbf{w}_{init}^{(t)}$ and $\alpha_{best}^{(t)}$, optimization of (~\ref{opt:main_opt}) converges very quickly as demonstrated in the experiments (\textit{in Table \ref{tab:ablation} of Appendix}). Please note that, \textit{we calculate the Hessian for only $n$ scalar parameters, with $n$ representing the number of source models. Typically, in common application domains, addressing distribution shifts requires only a small number of source models, making the computational overhead of calculating hessian negligible}.

Please refer to Algorithm~\ref{algo1} for a complete overview of \shortname{}. 
%
\subsection{Theoretical insights regarding combination weights}
\begin{theorem}[Convergence of Optimization~\ref{opt:main_opt}.] The Optimization ~\ref{opt:main_opt} converges according to the rule as follows:
\begin{equation}
    \frac{1}{(k+1)} \sum_{j=0}^k \| \nabla_{\aleph} \mathcal{L}_w (\mathbf{w^{(j)}}) \|_2^2 \leq \frac{2 (\mathcal{L}_w(w^{(0)}) - \mathcal{L}_w(w^\star))}{\alpha_{best}^{(t)} (k+1)}
\end{equation}
where, $\nabla_{\aleph}$ represents the gradient of the objective function over the set of n-simplex $\aleph$ and $j$ represents the iteration number.

\label{Thm:convergence}

\end{theorem}

\begin{proof}
    Please refer to the Appendix (Section \ref{sec:proof}) for the proof.
\end{proof}

\noindent \textbf{Implication of Theorem \ref{Thm:convergence}.} The theorem tells us that to make the optimization converge faster with fewer iterations (small $k$), it is crucial to start with a good initialization close to the best solution ($(\mathcal{L}(w^{(0)}) - \mathcal{L}(w^\star))$ should be small). By using Eqn.~\eqref{w_init}, we ensure this condition for quicker convergence. Also, please note that in Theorem \ref{Thm:convergence}, $j$ denotes the iteration number in the optimization process, and for simplicity, the batch number $t$ has been intentionally omitted from the notation. 

\label{sec:algorithm}
\begin{algorithm}[h!]
\footnotesize
\SetAlgoLined
\textbf{Input:} Pre-trained source models $\{\mathrm{f}^j_S\}_{j=1}^N$, streaming sequential unlabeled test data $\{x_i^{(1)}\}_{i=1}^B\rightarrow \{x_i^{(2)}\}_{i=1}^B\rightarrow \ldots \{x_i^{(t)}\}_{i=1}^B \rightarrow \ldots$  \\
\textbf{Output:} Optimal inference model for $t$-th test batch $\mathrm{f}_T^{(t)} \ \forall t$\\
\textbf{Initialization:} Assign $\mathrm{f}^{j(1)}_S \leftarrow \mathrm{f}^j_S \ \forall j$ \\
\While{$t \geq 1$}{
 Set initial $\mathbf{w}_{init}^{(t)}$ using Eqn.~\eqref{w_init}\\
 Set $\alpha_{best}^{(t)}$ using Eqn.~\eqref{alpha} \\
 Solve optimization~\ref{opt:main_opt} to get $\mathbf{w}^{\star(t)}$\\
 Infer the test batch $t$ using inference model $\mathrm{f}_T^{(t)}$ using Eqn.~\eqref{model_inf} \\
 Find source index $k$ such that $k= \text{arg} \ \underset{1\leq j\leq N}{\text{max}} \ \mathrm{w}^{\star(t)}_j$ \\
 Update source model $\mathrm{f}_S^{k(t)}$ according to Model Parameter Update paragraph of Section \ref{sec:learning_comb_weight} to get $\overline{\mathrm{f}_S^{k(t)}}$ \\
 \For{$1 \leq j \leq N$}{
\If{$j=k$}{
Set $\mathrm{f}^{j(t+1)}_S \leftarrow \overline{\mathrm{f}_S^{j(t)}}$ \\
\Else{Set $\mathrm{f}^{j(t+1)}_S \leftarrow \mathrm{f}_S^{j(t)}$}
}
}
}
\caption{Overview of \shortname} 
\label{algo1}
\end{algorithm}
\vspace{-1em}

\subsection{Theoretical insights regarding model update}
We now provide theoretical justification on how \shortname~ selects the best source model by optimally trading off model accuracy and domain mismatch. At time $t$, let $\mathrm{f}_S^{(t)}$ be the set of source models defined as $\left[\mathrm{f}_S^{1(t)}~\mathrm{f}_S^{2(t)}~\dots~\mathrm{f}_S^{N(t)}\right]$. \shortname~ aims to learn a combination of these models by optimizing weights $w$ on the target domain. For simplicity of exposition, we consider convex combinations $w\in\Delta$ where $\Delta$ is the $N$-dimensional simplex.

To learn $\w\in\Delta$, \shortname{} runs empirical risk minimization on the target task using a loss function $\ell(\cdot)$ with pseudo-labels generated by $\w$-weighted source models. Let $\Lc(f)$ denote the target population/test risk of a model $f$ (with respect to ground-truth labels) and $\lopt$ represent the optimal population risk obtained by choosing the best possible $w\in\Delta$ (i.e.~oracle risk). We introduce the functions: \textbf{(1)} $\Psi$ which returns the distance between two data distributions and \textbf{(2)} $\varphi$ which returns the distance between two label distributions. We note that, rather than problem-agnostic metrics like Wasserstein, our $\Psi,\varphi$ definitions are in terms of the loss landscape and source models $\mathrm{f}_S^{(t)}$, hence tighter. We have the following generalization bound at time step $t$ (precise details in Appendix Section \ref{sec:proof}).

\begin{theorem}
Consider the model $\mathrm{f}_T^{(t)}$ with combination weights $\w^{\star(t)}$ obtained via \shortname{} by minimizing the empirical risk over $B$ IID target examples per Eqn.~\ref{model_inf}. Let $\hat{y}^{(t)}_\w$ denote the pseudo-label variable of $\w$-weighted source models and $\mathcal{D}^{(t)}_\w=\sum_{i=1}^N \mathrm{w}_i^{(t)} \mathcal{D}_{S_i}^{(t)}$ denote weighted source distribution. Under Lipschitz $\ell$ and bounded $\mathrm{f}_S^{(t)}$, with probability at least $1-3e^{-\tau}$ over the target batch, test risk obeys 
\begin{align*}
\small
    \underbrace{\mathcal{L}(\mathrm{f}_T^{(t)})}_{\text{\shortname}} -\underbrace{\lopt}_{\text{Optimal}}\leq \underset{\mathrm{w} \in \Delta}{\text{min}} \{ \underbrace{\Psi(\mathcal{D}_T^{(t)}, \mathcal{D}^{(t)}_\w)}_{\textbf{shift}} + \underbrace{\varphi(\hat{y}_\w^{(t)},y_\w^{(t)})}_{\textbf{quality}}\}  +  \sqrt{\Tilde{\mathcal{O}}((N+\tau)/B)}. 
\end{align*}
\label{theorem:thm1}
\end{theorem}
\begin{proof}
    Please refer to the Appendix (Section \ref{sec:proof}) for the proof.
\end{proof}



\paragraph{Discussion.} In a nutshell, this result shows how \shortname{} strikes a balance between: (1) choosing the domain that has the smallest \textbf{shift} from target, and (2) choosing a source model that has high-\textbf{quality} pseudo-labels on its own distribution (i.e.~$\hat{y}_\w^{(t)}$ matches $y_\w^{(t)}$). From our analysis, it is evident that, rather than adapting the source models to the target distribution, if we simply optimize the combination weights to optimize pseudo-labels for inference, the left side excess risk term ($\mathcal{L}(\mathrm{f}_T^{(t)})-\lopt$) becomes upper bounded by a relatively modest value. This is because the \textbf{shift} and \textbf{quality} terms on the right-hand side are optimized with respect to $\mathrm{w}$. We note that $\sqrt{N/B}$ is the generalization risk due to finite samples $B$ and search dimension $N$.

To further refine this, our immediate objective is to tighten the upper bound. This can be achieved by individually adapting each source model to the current test data, all the while maintaining the optimized $\mathrm{w}$ constant. Yet, such an approach is not ideal since our second goal is to preserve knowledge from the source during continual adaptation.
To attain our desired goal, we must relax the upper bound, reducing our search over $\mathrm{w} \in \hat{\Delta}$. Here, $\hat{\Delta}$ is the discrete counterpart of the simplex $\Delta$. The elements of $\hat{\Delta}$ are one-hot vectors that have all but one entry zero. The elements of $\hat{\Delta}$ essentially represent discrete model selection. Examining the main terms on the right reveals that: (i) source-target distribution shift and (ii) divergence between ground-truth and pseudo-labels are all minimized when we select the source model with the highest correlation to target. This model, denoted by $\f^{\star(t)}_S$,  essentially corresponds to the largest entry of $\w^{\star(t)}$ and presents the most stringent upper bound within the $\hat{\Delta}$ search space. Thus, to further minimize the right hand side, the second stage of \shortname{} adapts $\f^{\star(t)}_S$ with the current test data. Crucially, besides minimizing the target risk, this step helps avoids forgetting the source because $\f^{\star(t)}_S$ already does a good job at the target task. Thus, during optimization on target data, $\f^{\star(t)}_S$ will have small gradient and will not move much, resulting in smaller forgetting. Please refer to the Appendix (Section \ref{sec:proof}) for more detailed discussion along with the proof of this theorem.

\section{Experiments}


\noindent \textbf{Datasets.} 
We demonstrate the efficacy of our approach using both \textbf{static target distribution} and \textbf{dynamic target data distributions}. For static case, we employ the \textit{Digits} and \textit{Office-Home} datasets \cite{venkateswara2017deep}. For the dynamic case, we utilize  
\textit{CIFAR-100C} and \textit{CIFAR-10C} \cite{hendrycks2019benchmarking}. Detailed descriptions of these datasets along with results on segmentation task can be found in the Appendix.

\noindent\textbf{Baseline Methods.} As our problem setting is most closely related to test time adaptation, our baselines are some widely used state-of-the-art (SOTA) single source test time adaptation methods: we specifically compare our  algorithm with Tent \cite{wang2020tent}, CoTTA \cite{wang2022continual} and EaTA \cite{niu2022efficient}. These methods deal with adaptation from small batches of streaming data and without the source data, which is our setting, and hence we compare against these as our baselines. To evaluate the adaptation performance, we follow the protocol similar to \cite{ahmed2021unsupervised}, where we apply each source model to the test data from a particular test domain individually, which yields X-Best and X-Worst where ``X" is the name of the single source adaptation method, representing the highest and lowest performances among the source models adapted using method ``X", respectively. For our algorithm, we extend all of the methods ``X" in the multi source setting and call the multi-source counterpart of ``X" as ``X+\shortname". 

\noindent \textbf{Implementation Details.} We use ResNet-18 \cite{he2016deep} model for all our experiments. For solving the optimization of Eq. \eqref{opt:main_opt}, we first initialize the combination weights using Eq. \eqref{w_init} and calculate the optimal learning rate using Eq. \eqref{alpha}. After that, we use 5 iterations to update the combination weights using SGD optimizer and the optimal learning rate. For all the experiments we use a batch size of 128, as used by Tent \cite{wang2020tent}. For more details on implementation and experimental setting see Appendix.

\noindent\textbf{Experiment on CIFAR-100C.} 
We conduct a thorough experiment on this dataset to investigate the performance of our model under dynamic test distribution. We consider 3 corruption noises out of 15 noises from CIFAR-100C, which are adversarial weather conditions namely \textit{Snow}, \textit{Fog} and \textit{Frost}. We add these noises for severity level $5$ to the original CIFAR-100C training set and train three source models, one for each noise. Along with these models, we also add the model trained on clean training set of CIFAR-100. During testing, we sequentially adapt the models across the 15 noisy domains, each with a severity of 5, from the CIFAR-100C dataset \cite{wang2022continual, niu2022efficient}. We report the results for the experiment in Table ~\ref{tab:cifar100}. Moreover, we also conduct an experiment on CIFAR10-C with the exact same experimental settings as with CIFAR100-C. \textit{CIFAR-10C results are in Table \ref{tab:cifar10} of Appendix}.

\input{tables/cifar100_new}

From the table, we can draw two key observations:\\
(i) As anticipated, X+\shortname \ consistently outperforms X-Best across each test distribution, underscoring the validity of our algorithmic proposition.

\input{tables/office_new}
(ii) Given that the CoTTA and EaTA methods are tailored to mitigate forgetting, the average error post-adaptation across the 15 noises using these methods is significantly lower than that of Tent, which is not designed for this specific challenge. For instance, in Table~\ref{tab:cifar100}, Tent-Best error is approximately 68.2\%, while CoTTA and EaTA-Best are around 39.9\% and 38.5\%, respectively. However, when these adaptation methods are incorporated into our framework, the final errors are remarkably close: 37.1\% for Tent, 34.2\% for EaTA, and 36.9\% for CoTTA. This suggests that even though Tent is more lightweight and faster compared to the other methods and is not inherently designed to handle forgetting, its performance within our framework is on par with the results obtained when incorporating the other two methods designed to prevent forgetting. This shows the generalizability of our approach to various single-source methods. Note that identical to the experiment on CIFAR100-C, the results on CIFAR10-C in Table \ref{tab:cifar10} follow the same trend where X+\shortname{} outperforms the X-Best.

\noindent\textbf{Experiment on Office-Home.}
\label{sec:office-home}
We report the results of the experiment on static distribution using the Office-Home dataset in Table~\ref{tab:office}. Each column in the table represents a target domain from Office-Home dataset. We train three source models on the remaining Office-Home datasets. For instance, in case of `Ar' column, `Ar' is the target domain where three source models trained on `Cl', `Pr' and `Rw' are adapted in test time. We calculate the test error of each incoming test batch and then report the numbers by averaging the error values over all the batches. The table shows that \shortname \ provides a significant reduction of test error compared to the best single source model. This demonstrates that when presented with an incoming test batch, \shortname \ has the capability to effectively blend all available sources using optimal weights, resulting in superior performance compared to the best single source model. It is important to note that each test batch in this experiment is drawn from the same stationary distribution, which represents the distribution of the target domain. We conduct a similar experiment with the same experimental settings on Digits dataset that can be found in Table \ref{tab:digit} of Appendix.

\noindent\textbf{Analysis of Forgetting.} Here, we demonstrate the robustness of our method against catastrophic forgetting by evaluating the classification accuracy on the source test set after completing adaptation to each domain \cite{niu2022efficient, song2023ecotta, chakrabarty2023sata}. For \shortname, we use our ensembling method to adapt to the incoming domain. After adaptation, we infer each of the adapted source models on its corresponding source test set. For the baseline single-source methods, every model is adapted individually to the incoming domain, followed by inference on its corresponding source test set. The reported accuracy represents the average accuracy obtained from each of these single-source adapted models. 


From Figure~\ref{fig:forgetting}, we note that our method consistently maintains its source accuracy during the adaptation process across the 15 sequential noises. In contrast, the accuracy for each individual single-source method (X) declines on the source test set as the adaptation process progresses. Specifically, Tent, not being crafted to alleviate forgetting, experiences a sharp decline in accuracy. While CoTTA and EaTA exhibit forgetting, it occurs at a more gradual pace. Contrary to all of these single-source methods, our algorithm exhibits virtually no forgetting throughout the process. 

\textbf{Ablation Study.}
We conduct an ablation study in Tables \ref{tab:ablation}, \ref{tab:ablation_entropy_init} in the Appendix to evaluate the impact of various initialization and learning rate strategies on the optimization process described in~(\ref{opt:main_opt}). Our findings demonstrate that the initialization and learning rate configurations generated by our method outperform other alternatives. 

\begin{wrapfigure}{r}{0.5\textwidth} 
    \centering
    \vspace{-2em}
    \includegraphics[width=0.5\columnwidth]{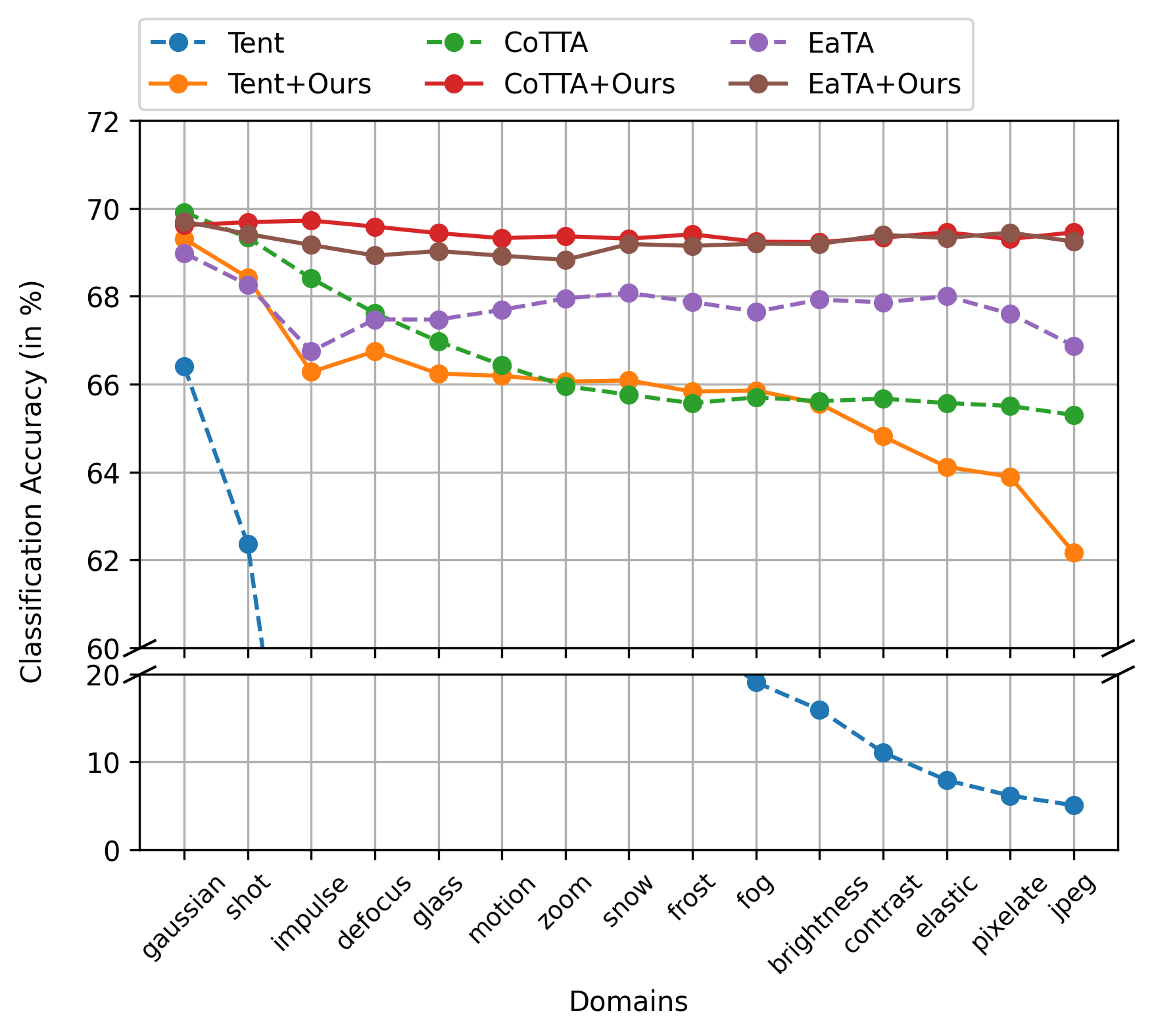} 
    \caption{\textbf{Comparison with baselines in terms of source knowledge forgetting.} Maintaining the same setting as in Table~\ref{tab:cifar100}, we demonstrate that by integrating single-source methods with \shortname, the source knowledge is better preserved during dynamic adaptation. Unlike all these single-source methods, our algorithm demonstrates virtually no forgetting throughout the entire adaptation process.} 
    \vspace{-1.5em}
    \label{fig:forgetting}
\end{wrapfigure}

Additionally, our experiments in Tables \ref{tab:cotta_ablation}, \ref{tab:update} and \ref{tab:ablation_weighted_update} in the Appendix reveal that selectively updating the most correlated model parameters enhances performance compared to updating all model parameters, the least correlated ones, a selected subset of correlated models or even updating the models according to their combination weights. We report the comparison with MSDA in Table \ref{tab:ablation_msda} and Model-Soups in Table \ref{tab:ablation_model_soup}. We also report the values of the combination weights learned by our method. See Section \ref{sec:ablation_appendix} of the Appendix for detailed observations. 


\section{Conclusions} We propose a novel framework called \shortname, that effectively combines multiple source models during test time with small batches of streaming data and without access to the source data. It achieves a test accuracy that is at least as good as the best individual source model. In addition, the design of \shortname{} offers the added benefit of naturally preventing the issue of catastrophic forgetting. To validate the effectiveness of our algorithm, we conduct experiments on a diverse range of benchmark datasets for classification and semantic segmentation tasks. We also demonstrate that \shortname{} can be integrated with a variety of single-source methods. Theoretical analysis of the performance of \shortname{} is also provided.

\section{Broader Impact and Limitations} Implementing multiple models for adaptation on dynamic distribution can yield broad impacts. For instance, this approach could find applications in robot navigation, autonomous vehicles or decision making in dynamically evolving scenarios. In all these cases, the algorithm can intelligently select the optimal combination of models during inference, ensuring sustained performance over extended periods. Our method currently assumes that data sampled within a batch comes from the same distribution. In the future, we aim to explore using mixed data samples from different target domains within a batch.


\section*{Acknowledgments}
The work was partially supported under NSF grant CCF-2008020. Additionally, research was sponsored by the OUSD (R\&E)/RT\&L and was accomplished under Cooperative Agreement Number W911NF-20-2-0267. The views and conclusions contained in this document are those of the authors and should not be interpreted as representing the official policies, either expressed or implied, of the ARL and OUSD(R\&E)/RT\&L or the U.S. Government. The U.S. Government is authorized to reproduce and distribute reprints for Government purposes notwithstanding any copyright notation herein. This work was also partially supported by the Laboratory Directed Research and Development (LDRD) Program (25-006) of Brookhaven National Laboratory under U.S. Department of Energy Contract No. DE-SC0012704.

\bibliography{example_paper}
\bibliographystyle{unsrtnat}

\newpage
\appendix
\onecolumn

\section*{Appendix Overview:}
    
\newcommand\DoToC{%
  \startcontents
  \printcontents{}{1}{\textbf{Contents}\vskip3pt\hrule\vskip5pt}
  \vskip3pt\hrule\vskip5pt
}
\DoToC
\newpage

\section{Proof and discussion of Theorems 1 and 2}
\label{sec:proof}
\begin{proof}[Proof of Theroem~\ref{Thm:convergence}] The optimization \eqref{opt:main_opt} has a structure similar to a class of non convex problems as follows:

\begin{mini}|l|
{x \in \chi}{g(x)-h(x))}{}{}
\label{opt:proof}
\end{mini}

where $\chi$ is a closed convex set, $g(x)$ is  $M_g$ smooth and $h(x)$ is a continuous convex function. In such cases, the optimization converges as follows \cite{khamaru2018convergence}:

\begin{equation}
    \frac{1}{(k+1)} \sum_{j=0}^k \left(\nabla_{\chi} \|f(x^k)\|_2^2 \right) \leq \frac{2 (f(x^0) - f^\star))}{\alpha (k+1)}
\end{equation}

where, $f(x)=(g(x)-h(x))$.

In our case $g(x)=c$, where $c$ is  a constant (smooth and continous) and $h(x)$ is negative of the Shannon entropy, which is continous and convex. Also, $\chi$ is the n-simplex $\aleph$, which is a closed convex set. So, according to the proof derived in \cite{khamaru2018convergence}, we can conclude the bound in Theroem~\ref{Thm:convergence}.
\end{proof}

\begin{proof}[Proof of Theorem~\ref{theorem:thm1}]
    We adapt the theorem from a corollary (corollary 1) in \cite{oymak2021generalization}. In this corollary the following result was derived:
\begin{align*}
    \mathcal{L}(\mathrm{f}_{\hat{\alpha}}^{\tau}) \leq \underset{\alpha \in \Delta}{\text{min}}(l_{\star}^{\alpha}(\mathcal{D}) + &\text{DM}_{\mathcal{D}^{\prime}}^{\mathcal{D}}(\alpha) + 4\Gamma \mathcal{R}_{n_{\tau}}(\mathcal{F}_{\alpha})) \\ &+ \sqrt{\Tilde{\mathcal{O}}((h_{eff}+t)/n_{\mathcal{\nu}})} +\delta
\end{align*}

Here $\mathrm{f}^{\tau}$ in the $\mathrm{f}_{\hat{\alpha}}^{\tau}$ is the trained model on the training($\tau$) distribution $\mathcal{D}^{\prime}$ and $\hat{\alpha}$ is a hyper-parameter that has been empirically optimized by fine tuning on the validation($\mathcal{\nu}$) distribution $\mathcal{D}$. $\mathcal{L}$ is the expected risk over the distribution $\mathcal{D}$. $\text{DM}$ measures the distribution mismatch via difference of sub-optimality gap using the training and validation distribution. $\mathcal{R}_{n_{\tau}}(\mathcal{F}_{\alpha}))$ is the Rademacher complexity of the function class $\mathcal{F}$ with $\alpha$ as the hyper-parameter. The corollary holds for probability of at least $1-3e^{-t}$ and $h_{eff}$ is the effective dimension of the hyper-parameter space. Also $n_{\mathcal{\nu}}$ is the number of samples under the validation. The bound can be first of all easily extended to the source/target scenario instead of train/validation. In our scenario the source models jointly construct the function class $\mathcal{F}_{\alpha}$ where, the hyper-parameter $\alpha$ is the combination weight $\mathrm{w}$. Effective dimension for our case is exactly the number of source model $N$ and instead of $t$ we took $\tau$ as the probability variable. For the sake of simplicity we omitted $\delta>0$  which is a positive constant along with the Rademacher complexity. Also $n_{\mathcal{\nu}}=B$ in our setting since we have $B$ number of samples for the target/validation. Now there is a new term in our bound which is $\varphi$ which was not in the original corollary. This term is used to account for the mismatch between actual and pseudo-labels generated by the source. This is done due to the fact that we do empirical minimization of the entropy of the target pseudo-label since the problem is unsupervised and actual labels are not available. The left side of the inequality is derived using the test/target pseudo-label. Consequently, we can introduce an added distribution mismatch term. This term can be broken down into three components: mismatch from target pseudo to target ground truth (gt), from target gt to source gt, and from source gt to source pseudo label. Of these components, the first two can be readily integrated into the $\Psi(.)$ function , given that it measures the discrepancy between the weighted source and the target. The remaining third component is denoted by the $\varphi(.)$ function. This completes the proof.
\end{proof}

\section{Results on Digits}
\label{sec:digits}
We report here the results of digit classification in Table~\ref{tab:digit}. Similar to the experiment on Office-Home dataset, each column of the table represents a target domain dataset. We train four source models on the rest of the digit datasets. For instance, in case of `MM' column `MM' is the target domain which is adapted using four source models trained on `MT', `UP', `SV' and `SY' respectively.

\input{tables/digit_new}

We once again calculate the test error for each incoming test batch and report the results by averaging the errors across all batches. The table demonstrates that \shortname\ achieves a significant reduction in test error compared to the best single source (on average 3\% error reduction than the best source). Another baseline exists that simply uses a naive ensemble of the source models, without any weight optimization. In situations where there's a significant performance gap between the best and worst source models adapted using single-source methods, a uniform ensemble of these models produces a predictor that trails considerably behind the best-adapted source, as noted by \cite{ahmed2021unsupervised}. Referring to Table~\ref{tab:digit}, when testing on the SVHN dataset, the error disparity between the best and worst adapted source models is approximately 70.7\%—a substantial margin. Consequently, using a uniform ensemble in such a scenario results in an error rate of roughly 45.5\% (experimentally found, not reported in the table). This is strikingly higher than our method's error rate of around 14.2\%. \textit{Given these findings, we deduce that uniform ensembling is not a reliable approach for model fusion. Thus, we exclude it from our experiment section's baseline}.




\section{Results on CIFAR-10C}
\label{sec:cifar10c}
Here, we report the results on dynamic target distribution using CIFAR-10C dataset. Note that identical to the experiment on CIFAR100-C in the main paper the results on CIFAR10-C in Table \ref{tab:cifar10} follow the same trend where X+\shortname{} outperforms the X-Best.

\input{tables/cifar10_new}

In the single-source scenario, one among the four source models achieves the X-Best (for example CoTTA-Best) accuracy for a specific domain. The determination of which individual model (from the four) will attain the best accuracy for that domain remains uncertain beforehand. Furthermore, the individual source model yielding the X-Best accuracy varies across different domains within CIFAR10-C. However, in our X+CONTRAST approach, the need to deliberate over the selection of one out of the four source models is eliminated. X+CONTRAST reliably outperforms any single source X-model that might achieve the X-Best accuracy.

Individual TTA methods may have distinct advantages. For example, Tent offers several distinct advantages over CoTTA, including its lightweight nature and faster performance. Conversely, CoTTA presents certain benefits over Tent, such as increased resilience against forgetting. Consequently, the choice between TTA methods is dependent on the user's preferences, aligning with the specific task at hand. In this experiment, we have demonstrated that \shortname{} can be integrated with any TTA method of the user's choosing.

\section{Ablation Study}
\label{sec:ablation_appendix}

\subsection{Initialization and Learning Rate}
\input{tables/ablation}
Table~\ref{tab:ablation} presents the error rate results on the Office-Home dataset under the same experimental setting as Table~\ref{tab:office} (Appendix) with Tent as the adaptation method, but with different initialization and learning rate choices for solving the optimization in~(\ref{opt:main_opt}). It is evident from the table that our chosen initialization and adaptive learning rate result in the highest accuracy gain.

We additionally show another ablation study in Table \ref{tab:ablation_entropy_init}, where we initialize the combination weights based on the probability of source model predictions. More precisely, we set the initial weights inversely proportional to the entropy of the source model predictions. In simpler terms, a source model with low entropy receives a higher weight, while one with high entropy receives a lower weight.

\input{tables/ablation_entropy_init}

In the presented table for CIFAR-100C, we note a 16.5\% reduction in error resulting from our initialization method. We found that initializing the combination weights using the entropy of the test batch for various sources leads to somewhat uniform initialization. However, when we initialize the combination weights using KL divergence, we achieve a highly effective and peaky prior, favoring the most correlated source model with relatively higher weightage. This clarifies why initializing with entropies fails to converge quickly to the optimum, resulting in significantly poorer outcomes compared to our method.

\subsection{Model Update Policy}
\label{sec:model_update}
In Table \ref{tab:cotta_ablation} and ~\ref{tab:update}, we demonstrate that by updating only the model with the highest correlation to the target domain, our method produces the lowest test accuracy. This is in comparison to scenarios where we either update all models or solely the least correlated one. This empirical observation directly supports our theoretical assertion from the theorem: updating the most correlated model is most effective in preventing forgetting, thereby resulting in the smallest test error during gradual adaptation. We also experiment with another model update policy where a subset of model is updated.

\input{tables/ablation_cotta}

\input{tables/ablation_model_update}

\subsubsection{Subset of Model Update}
In this approach, rather than focusing solely on the most correlated source model, we identify and update a subset of source models that exhibit higher correlation than the rest of the models. Specifically, we select models for updating based on their combination weights, choosing only those whose weights exceed $1/n$, with $n$ representing the total number of models. The intuition behind selecting this threshold $1/n$ for subset selection is grounded in the distance of the combination weight distribution with respect to the uniform distribution. A uniform combination weight implies that all models are equidistant w.r.t the test distribution and should be updated. However, if only one model weight surpasses $1/n$, it signifies that only one model exhibits a high correlation with the overall model. 

Results are shown in Table \ref{tab:cotta_ablation} and ~\ref{tab:update}. Several key observations can be extracted from here. Notably, when utilizing the Tent adaptation algorithm, updating a subset of models results in significantly poorer performance compared to updating only the most correlated model. Conversely, with the CoTTA adaptation algorithm, the performance decrement from updating a subset of models is relatively minor compared to updating the most correlated model. This discrepancy can be attributed to the varying degrees of resistance to forgetting exhibited by these adaptation algorithms. Updating multiple models tends to induce forgetting, leading to a decline in overall performance, especially when the adaptation algorithm is not highly resistant to forgetting. Despite the adaptation method's robustness to forgetting, it has been consistently observed that updating the most correlated model not only delivers superior performance but also offers computational advantages over updating a subset of models. This approach simplifies the update process and ensures more efficient use of computational resources.

\subsubsection{Model Update According to Weight}

 Here, we update the model $j$ weighted by $w_{j}$. To do so, we need to properly devise an approach that updates models in measures according to their correlation with the test data. Drawing inspiration from recent studies that employ variable learning rates for single-source TTA, we devise a strategy to adjust the learning rate $\eta_{j}$ used in updating model $j$ based on their respective combination weights $w_{j}$. Specifically, we assigned the highest learning rate $\eta_{max} = 0.001$ ($0.001$ is the learning rate used for both Tent and CoTTA in our experiments) to the model with the greatest combination weight, while the lowest learning rate $\eta_{min} = 0.0001$, (a tenfold reduction) was allocated to the model with the lowest combination weight. For the remaining models, we interpolated their learning rates proportionally between the highest and lowest rates, based on their respective combination weights following the formula: $\eta_{j} = [\left(\frac{w_{j} - w_{min}}{w_{max} - w_{min}}\right) \times (\eta_{\text{max}} - \eta_{\text{min}})] + \eta_{\text{min}} $. In the Table ~\ref{tab:ablation_weighted_update}, we present the resulting error rates for CIFAR-100C dataset using both Tent and CoTTA.

\input{tables/ablation_weighted_model_update}

 Our investigation reveals that, in scenarios where the update algorithm exhibits limited robustness against forgetting, such as Tent, updating only the model with the highest combination weight proves more advantageous. This is because even marginal updates to uncorrelated models can lead to detrimental forgetting, resulting in poor performance. Conversely, when the update algorithm demonstrates resilience against forgetting (CoTTA), updating the most correlated model impacts performance the most. While updating uncorrelated models does not substantially enhance performance, it significantly increases computational costs. It should also be noted that we have found exactly same finding with our ablation study focused on updating subsets of models. Consequently, we assert that updating the single model with the highest combination weight yields optimal performance across all scenarios.

\subsection{Combination Weight Visualization}
To provide insight into the combination weight distribution, let's consider an example where the source models are trained on the clean, snow, frost, and fog domains using the training data. We then select one of these domains to collect the average weights over all the test data. When the test data is from the fog domain, the weight distribution appears as follows: [0.05, 0.08, 0.09, 0.78]. On the other hand, when the test domain is frost, we observe the following weight distribution: [0.07, 0.14, 0.69, 0.11]. These results clearly illustrate that the weight distribution accurately reflects the correlation between the source models and target domains.

\subsection{Comparison with MSDA}
Existing multi-source source-free methods are designed for offline settings where all the target data are available during adaptation. However, in our setting, data is received batch by batch during adaptation. Therefore, theoretically, these methods are expected to perform worse in our setup. Nevertheless, we compared \shortname{} with the seminal paper \cite{ahmed2021unsupervised} on source-free multi-source Unsupervised Domain Adaptation (UDA), specifically the DECISION method, to demonstrate its effectiveness in an online adaptation setting. We keep the hyperparameters exactly the same as described in the DECISION and perform adaptation on each incoming batch of test data with the number of epochs specified in DECISION.

\input{tables/ablation_msda}

It is evident from Table \ref{tab:ablation_msda} that DECISION performs notably poorly in the online setting, with an error rate almost 56\% higher than \shortname{}. DECISION utilizes clustering of the entire offline dataset based on the number of classes, a method not feasible to accurately implement in our setting with very small batch sizes. This highlights the necessity of a multi-source method specifically tailored for our setting.

\subsection{Comparison with Model Soups}
Model Soups \cite{wortsman2022model} is a popular approach for utilizing a set of models by averaging their parameters to create a single model for inference on test data. For completeness, we compare our method against Model Soups.

\input{tables/ablation_model_soup}

As shown in Table \ref{tab:ablation_model_soup}, the performance of Model Soups is significantly worse compared to our method. Model Soups averages the parameters of models fine-tuned on the same data distribution. However, in our setting, we have models trained on different source domains, making the averaging of model parameters suboptimal.

\section{Implementation Details}
\label{sec:implementation_details}
In this section, we provide a comprehensive overview of our experimental setup. We conducted two sets of experiments: one on a stationary target distribution, and the other on a dynamic target distribution that changes continuously. The reported results in the main paper are average of three runs with different seeds.

\subsection{Stationary Target}
\subsubsection{Digit Classification}
The digit classification task consists of five distinct domains from which we construct five different adaptation scenarios. Each scenario involves four source models, with the remaining domain treated as the target distribution. In total, we construct five adaptation scenarios for our study. 

The ResNet-18 architecture was used for all models, with an image size of $64\times 64$ and a batch size of 128 during testing. Mean accuracy over the entire test set is reported in Table \textcolor{red}{2} of the main paper. For Tent we use a learning rate of 0.01 and for rest of the adaptation method a learning rate of 0.001 is used. We use Adam optimizer for all the adaptation methods. Model parameter update is performed using a single step of gradient descent.

\subsubsection{Object Recognition}
The object recognition task on the Office-Home dataset comprises of four distinct domains from which we construct four different adaptation scenarios, similar to the digit classification setup. We use the same experimental settings and hyperparameters as the digit classification experiment, with the exception of the image size, which is set to $224\times 224$ in this experiment. The results of this evaluation are reported in Table \textcolor{red}{3} of the main paper.

\subsection{Dynamic Target}
\subsubsection{CIFAR-10/100-C}  
In this experiment, we use four ResNet-18 source models trained on different variants of the CIFAR-10/100 dataset: 1) vanilla train set, 2) train set with added fog (severity = 5), 3) train set with added snow (severity = 5), and 4) train set with added frost (severity = 5). To evaluate the models, we use the test set of CIFAR-10/100C (severity = 5) and adapt to each of the domains in a continual manner. The images are resized to $224\times 224$. For all the adaptation methods, a learning rate of $0.001$ with Adam optimizer is used.

\section{Semantic Segmentation}
\label{sec:semantic_seg}
\input{tables/seg_static}

\input{tables/acdc_dynamic_5steps}

Our method is not just limited to image classification tasks and can be easily extended to other tasks like semantic segmentation (sem-seg). We assume access to a set of sem-seg source models $\{\mathrm{f}^j_S\}_{j=1}^N$, where each model classifies every pixel of an input image to some class. Specifically, $\mathrm{f}^j_S:\mathbb{R}^{H\times W}\rightarrow \mathbb{R}^{H\times W \times K}$, where $K$ is the number of classes. In this case, the entropy in Eqn. \textcolor{red}{3} of the main paper will be modified as follows:

\begin{equation}
    \mathcal{L}_{w}^{(t)}(\mathbf{w}) = -\mathbb{E}_{\mathcal{D}_T^{(t)}} \sum_{h=1}^H \sum_{w=1}^W \sum_{c=1}^K \hat{y}_{ihwc}^{(t)} \log(\hat{y}_{ihwc}^{(t)}) 
\label{w_ent_semseg}
\end{equation}

Where, $\hat{y}_{ihwc}^{(t)}$ is the weighted probability output corresponding to class $c$ for the  pixel at location $(h,w)$ at time-stamp $t$. We modify Eqn. \textcolor{red}{3} in the main paper, while keeping the rest of the framework the same. 

\subsection{Datasets}
We use the following datasets in our experiments:\\
$\bullet$ \textbf{Cityscapes:} Cityscapes \cite{cordts2016cityscapes} is a large-scale dataset that has dense pixel-level annotations for 30 classes grouped into 8 categories (flat surfaces, humans, vehicles, constructions, objects, nature, sky, and void). There are also fog and rain variants \cite{sakaridis2018model, hu2019depth} of the Cityscapes dataset, where the clean images of Cityscapes have been simulated to add fog and rainy weather conditions. \\ 
$\bullet$ \textbf{ACDC:} The Adverse Conditions Dataset \cite{sakaridis2021acdc} has images corresponding to fog, night-time, rain, and snow weather conditions. Also, the corresponding pixel-level annotations are available. The number of classes is the same as the evaluation classes of the Cityscapes dataset. 

\subsection{Experimental setup}
We use Deeplab v3+ \cite{chen2018encoder} with a ResNet-18 encoder as the segmentation model for all the experiments. We resize the input images to a size of $512 \times 512$. Following the conventional evaluation protocol \cite{cordts2016cityscapes}, we
evaluate our model on 19 semantic labels without considering the void label.

We first experiment in a static target distribution setting. Specifically, we train three source models on clean, fog, and rain train splits of Cityscapes. We then evaluate the models on the test set of each of the weather conditions of ACDC dataset using \shortname{} and baseline Tent models. We use a batch size of 16 and report the mean accuracy over all the test batches. Again, we have updated the combination weights of \shortname{} with SGD optimizer using 5 iterations. For updating the source model in \shortname{} that has the most correlation with the incoming test batch, we use the Adam optimizer with a learning rate of $0.001$ and updated the batch-norm parameters with one iteration. The baseline Tent models are also updated with the same optimizer and learning rate. The results in Table.~\ref{tab:seg_static} clearly demonstrate that \shortname{} outperforms all the baselines on test data from each of the adverse weather conditions.

We also evaluate our method in a dynamic test distribution setting, where we have sequentially incoming test batches from the four weather condition test sets of ACDC dataset. The test sequence includes 5 batches of Rain, followed by 5 batches of Snow, 5 batches of Fog, and finally 5 batches of Night. This sequence is repeated (with the same test images) for a total of 5 rounds. We report the mean accuracy over the 5 batches and include the results for the 1st, 3rd, and 5th rounds in Table~\ref{tab:seg_dynamic}. We use the same hyperparameters as in the dynamic setting of previous experiments with the exception that the batch-size is 16.

\subsection{Visualization}
In Fig. \ref{fig:segmentation}, we present the input images along with the corresponding predicted masks of the baseline models and \shortname{} from the last round. The figure contains rows of input image samples from the four different weather conditions of the ACDC dataset, in the order of rain, snow, fog, and night. \shortname{} is compared with baseline adaptation method Tent, and as shown in Fig. \ref{fig:segmentation}, it is evident that \shortname{} provides better segmentation results compared to the baselines visually. \\

\begin{figure*}[h!]
\setlength\tabcolsep{0.5pt}
\centering
\begin{tabular}{cccccc}
\includegraphics[width=1.2in]{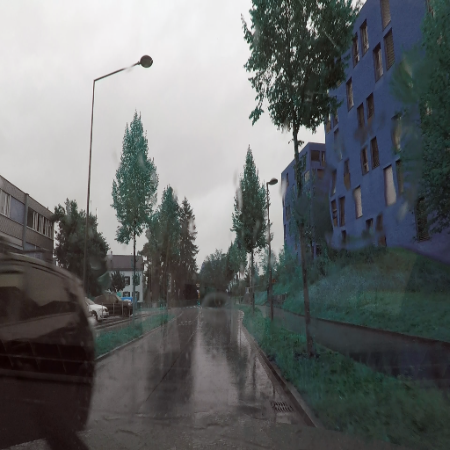} &
\includegraphics[width=1.2in]{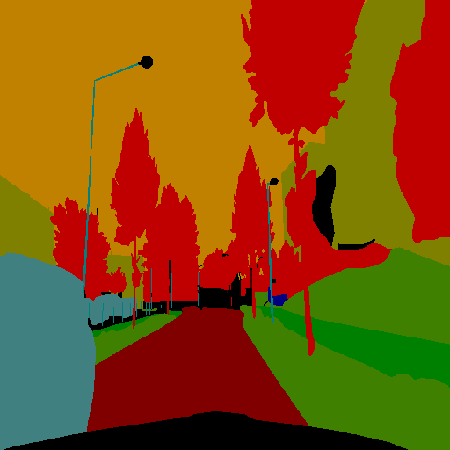} &
\includegraphics[width=1.2in]{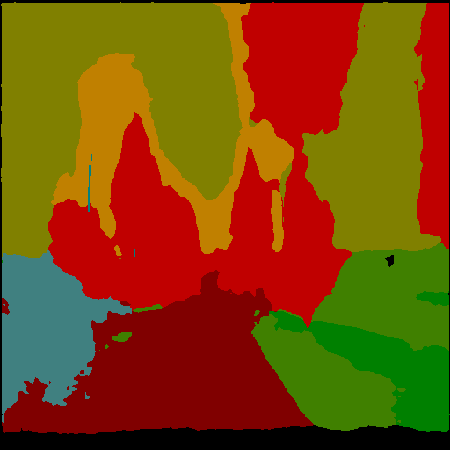} &
\includegraphics[width=1.2in]{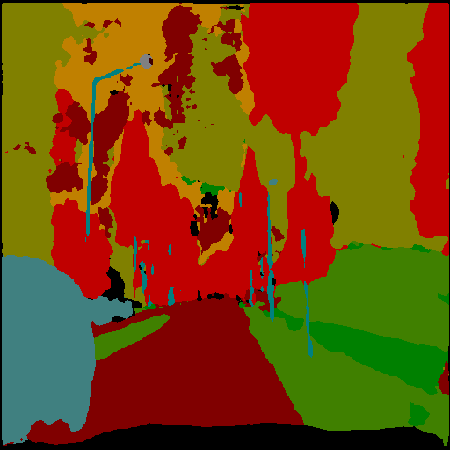}
\\

\includegraphics[width=1.2in]{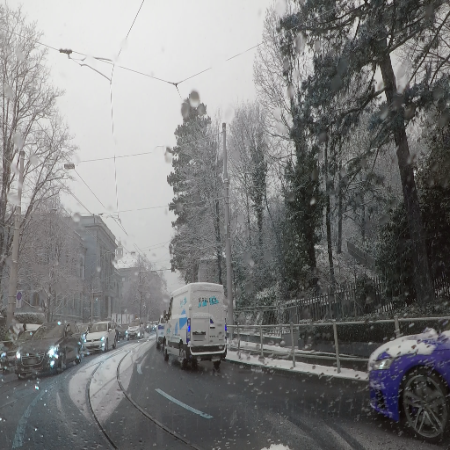} &
\includegraphics[width=1.2in]{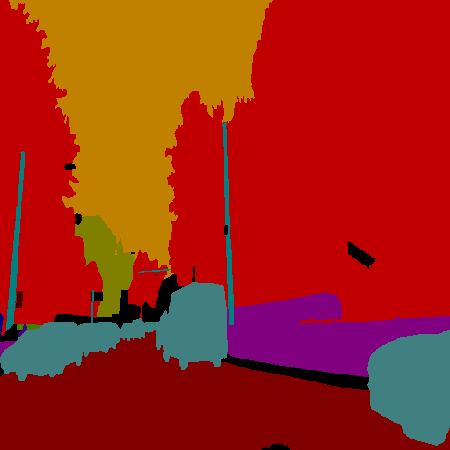} &
\includegraphics[width=1.2in]{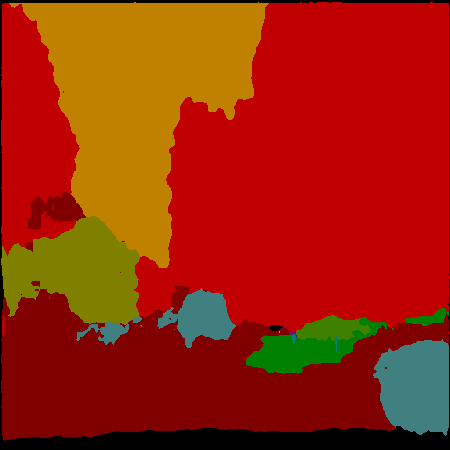} &
\includegraphics[width=1.2in]{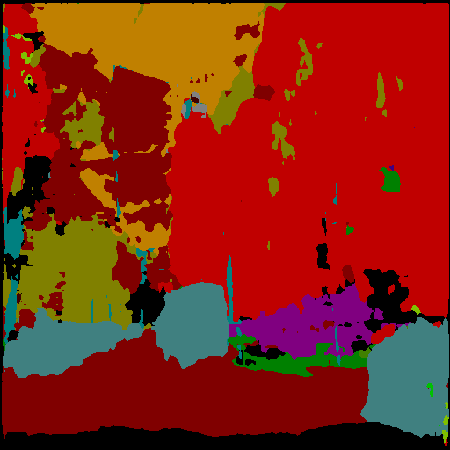}

\\
\includegraphics[width=1.2in]{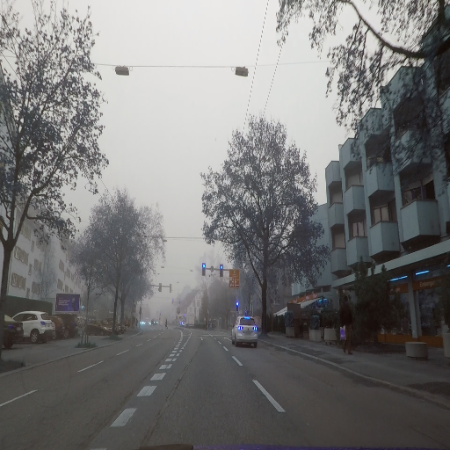} &
\includegraphics[width=1.2in]{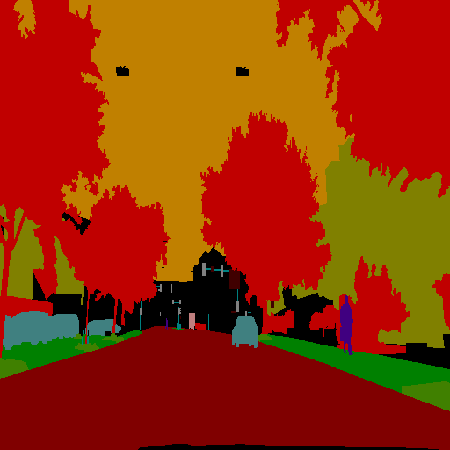} &
\includegraphics[width=1.2in]{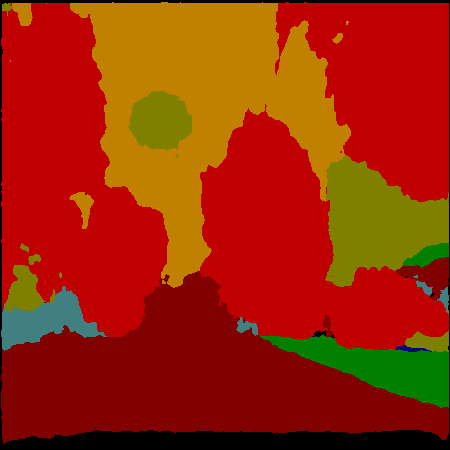} &
\includegraphics[width=1.2in]{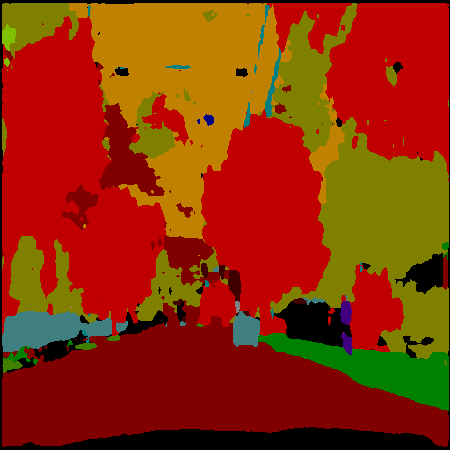}  
\\
\includegraphics[width=1.2in]{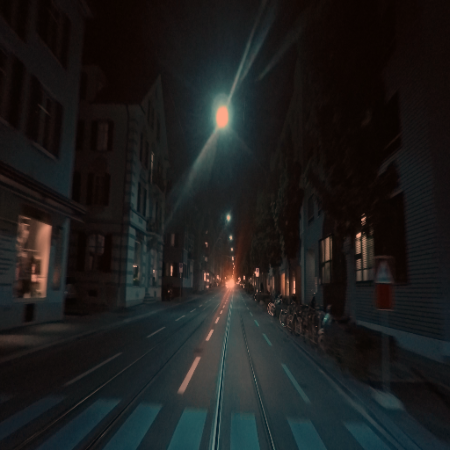} &
\includegraphics[width=1.2in]{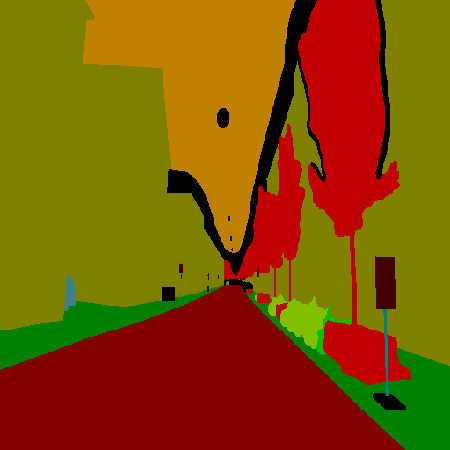} &
\includegraphics[width=1.2in]{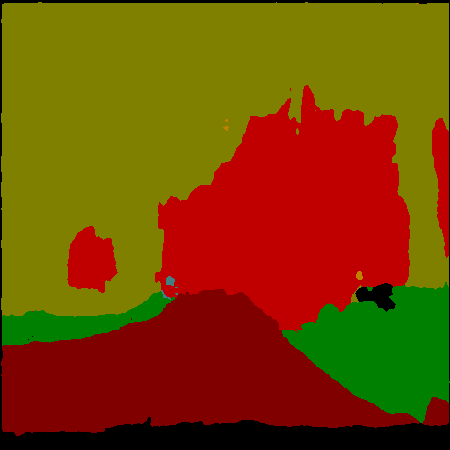} &
\includegraphics[width=1.2in]{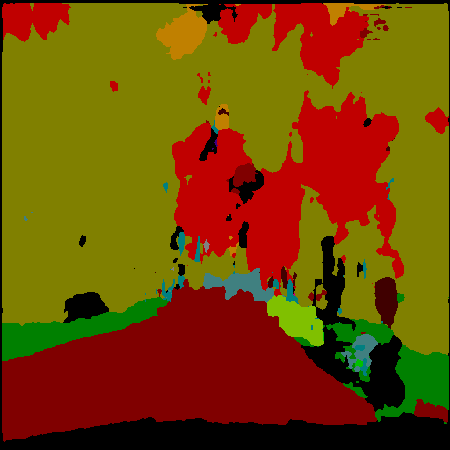}
\\
\textbf{\footnotesize{Input}} &
\textbf{\footnotesize{GT Mask}} &
\textbf{\footnotesize{Tent-Best}} &
\textbf{\footnotesize{\shortname}} 
\\
\end{tabular}
\caption{\textbf{Visual Comparison of \shortname{} with Baselines for Semantic Segmentation Task.} Each row in the figure corresponds to a different weather condition (rain, snow, fog, and night from top to bottom). It is evident that \shortname{} outperforms the baselines in terms of segmentation results.}
\label{fig:segmentation}
\end{figure*}

\section{Additional discussion}
\label{sec:add_discussion}
The $\varphi(.)$ function implies that trained sources should produce high-\textbf{quality} pseudo-labels within their own distribution. Essentially, this function evaluates the effectiveness of the model's training. For instance, even if the \textbf{shift} between the source and target is minimal, a poorly trained source model might still under-perform on the target. Observe that both the \textbf{shift} and the \textbf{quality} terms are minimized when we broaden our search space over $\hat{\Delta}$. This allows us to select a model that exhibits the highest correlation with the test domain, thereby providing us with the most strict bound within the discrete simplex. \\
Examining the issue through the lens of the gradient provides another perspective. By updating the source model that is most correlated with the test data, its gradient will be smaller than those of other models. Over time, this ensures that the model's parameters remain closer to the original source parameters, thereby preventing catastrophic forgetting. let's examine a toy case mathematically of the most correlated source can give us least gradient. \\
Let us assume a binary classification task with linear regression where the final activation is sigmoid $\sigma(.)$ function. Now let's take the pseudo-label for a sample $x$ be $\hat{y}$, where $\hat{y}= \sigma(w^{\top}x)$. Then the entropy $h$ of $\hat{y}$ will be $h = -\hat{y} \log(\hat{y})$. Then we take the derivative of the objective $h$ w.r.t $w$ weight as follows:

\begin{align*}
    & h = -\hat{y} \log(\hat{y}) \\
    \Rightarrow& \frac{\partial h}{\partial w} = (1 + \log(\hat{y})) \hat{y} (\hat{y}-1) x
\end{align*}

Now we can easily verify that if the source model is closest to the test domain, then the pseudo-label generated by the model has very small entropy which also means $\hat{y}$ is either close to $0$ or close to $1$. For both of the cases the derivative expression above goes close to zero which validate the claim of having smallest gradient for highest correlated source.

\section{KL divergence between two univariate Gaussians}
\label{sec:KL}

During the discussion of initialization of the combination weights in Section \textcolor{red}{3.5}, we come up with $\mathrm{\theta}_j^t$ which is calculated using the formula for KL divergence between two univariate Gaussians ${\cal{N}}(\mu_1,\sigma_1^2) \ \text{and} \ {\cal{N}}(\mu_2,\sigma_2^2)$. In this section, we provide the detailed derivation of this below: \\
From the definition of KL divergence, we know the distance between two distributions $p$ and $q$ is given by,
\begin{equation}
\begin{split}
     \mathcal{D}_{KL}(p,q) & = \int_{-\infty}^{+\infty} p(x)\log\left(\frac{p(x)}{q(x)}\right) \,dx  \\
     & = \int_{-\infty}^{+\infty} p(x)\log\left(p(x)\right) \,dx - \int_{-\infty}^{+\infty} p(x)\log\left(q(x)\right) \,dx
\end{split}
\label{KL_def}
\end{equation}

Here in this problem $p$ and $q$ are univariate Gaussians and can be expressed as follows:
\begin{align*}
p(x) = \frac{1}{(2 \pi \sigma_1^2)^{\frac{1}{2}}} \exp{\left({-\frac{(x-\mu_1)^2}{2\sigma_1^2}}\right)}, \quad \quad \quad 
q(x) = \frac{1}{(2 \pi \sigma_2^2)^{\frac{1}{2}}} \exp{\left({-\frac{(x-\mu_2)^2}{2\sigma_2^2}}\right)}.
\end{align*}

Now we compute the second term in Eqn.~\eqref{KL_def} as follows:

\begin{equation}
\begin{split}
     \int_{-\infty}^{+\infty} p(x)\log\left(q(x)\right) \,dx & = \log\left(\frac{1}{(2 \pi \sigma_2^2)^{\frac{1}{2}}}\right) - \int_{-\infty}^{+\infty} p(x) \frac{(x-\mu_2)^2}{2\sigma_2^2} \,dx \\
     & = \log\left(\frac{1}{(2 \pi \sigma_2^2)^{\frac{1}{2}}}\right) - \frac{\int_{-\infty}^{+\infty} x^2 p(x) \,dx -2 \mu_2\int_{-\infty}^{+\infty} x p(x) \,dx +\mu_2^2}{2\sigma_2^2} \\
     & = \log\left(\frac{1}{(2 \pi \sigma_2^2)^{\frac{1}{2}}}\right) - \frac{\mathbb{E}\left[X^2\right]  -2 \mu_2\mathbb{E}\left[X\right] +\mu_2^2}{2\sigma_2^2} \\
     & = \log\left(\frac{1}{(2 \pi \sigma_2^2)^{\frac{1}{2}}}\right) - \frac{\text{Var} \left[X\right] + \left(\mathbb{E}\left[X\right]\right)^2  -2 \mu_2\mathbb{E}\left[X\right] +\mu_2^2}{2\sigma_2^2} \\
     & = \log\left(\frac{1}{(2 \pi \sigma_2^2)^{\frac{1}{2}}}\right) - \frac{\sigma_1^2 + \mu_1^2  -2 \mu_2\mu_1 +\mu_2^2}{2\sigma_2^2} \\
     & = \log\left(\frac{1}{(2 \pi \sigma_2^2)^{\frac{1}{2}}}\right) - \frac{\sigma_1^2 + \left(\mu_1 -\mu_2\right)^2}{2\sigma_2^2} \\
\end{split}
\label{KL_def_2}
\end{equation}

In a similar manner we calculate the first term in
Eqn.~\eqref{KL_def} as follows:

\begin{equation}
\begin{split}
     \int_{-\infty}^{+\infty} p(x)\log\left(p(x)\right) \,dx & = \log\left(\frac{1}{(2 \pi \sigma_1^2)^{\frac{1}{2}}}\right) - \int_{-\infty}^{+\infty} p(x) \frac{(x-\mu_1)^2}{2\sigma_1^2} \,dx \\
     & = \log\left(\frac{1}{(2 \pi \sigma_1^2)^{\frac{1}{2}}}\right) - \frac{\int_{-\infty}^{+\infty} x^2 p(x) \,dx -2 \mu_1\int_{-\infty}^{+\infty} x p(x) \,dx +\mu_1^2}{2\sigma_1^2} \\
     & = \log\left(\frac{1}{(2 \pi \sigma_1^2)^{\frac{1}{2}}}\right) - \frac{\mathbb{E}\left[X^2\right]  -2 \mu_1\mathbb{E}\left[X\right] +\mu_1^2}{2\sigma_2^2} \\
     & = \log\left(\frac{1}{(2 \pi \sigma_2^2)^{\frac{1}{2}}}\right) - \frac{\text{Var} \left[X\right] + \left(\mathbb{E}\left[X\right]\right)^2  -2 \mu_1\mathbb{E}\left[X\right] +\mu_1^2}{2\sigma_1^2} \\
     & = \log\left(\frac{1}{(2 \pi \sigma_1^2)^{\frac{1}{2}}}\right) - \frac{\sigma_1^2 + \mu_1^2  -2 \mu_1^2 +\mu_1^2}{2\sigma_1^2} \\
     & = \log\left(\frac{1}{(2 \pi \sigma_1^2)^{\frac{1}{2}}}\right) - \frac{1}{2} \\
\end{split}
\label{KL_def_1}
\end{equation}

Now combining Eqn.~\eqref{KL_def_1} and Eqn.~\eqref{KL_def_2}, we get the final KL divergence as follows:

\begin{equation}
\begin{split}
     \mathcal{D}_{KL}(p,q) & = \log\left(\frac{1}{(2 \pi \sigma_1^2)^{\frac{1}{2}}}\right) - \frac{1}{2} -  \log\left(\frac{1}{(2 \pi \sigma_2^2)^{\frac{1}{2}}}\right) + \frac{\sigma_1^2 + \left(\mu_1 -\mu_2\right)^2}{2\sigma_2^2}\\
     & =  \log\left(\frac{\sigma_2}{\sigma_1}\right)  + \frac{\sigma_1^2 + \left(\mu_1 -\mu_2\right)^2}{2\sigma_2^2} - \frac{1}{2} \\
\end{split}
\label{KL_def_final}
\end{equation}

\section{Optimal step size in approximate Newton's method}
\label{sec:newton}
In the main paper, we compute the optimal combination weights by solving the optimization below:
\begin{mini}|l|
{\mathbf{w}}{\mathcal{L}_w^{(t)}(\mathbf{w})}{}{}
\addConstraint {\mathrm{w}_j \geq 0, \forall j \in \{1, 2, \dots, N\}}
\addConstraint {\sum_{j=1}^n \mathrm{w}_j=1}
\label{opt:main_opt_supple}
\end{mini}

To solve this problem, we begin by initializing $\mathbf{w}_{init}^{(t)}$ as $\delta(-\mathbf{\theta}^t)$. Next, we determine the optimal step size based on the initial combination weights to minimize the loss $\mathcal{L}_w^{(t)}$ as much as possible. Specifically, we use a second-order Taylor expansion to approximate the loss at the updated point after taking a single step with a step size of $\alpha^{(t)}$. Thus, after one step of gradient descent, the updated point becomes:

\begin{equation}
    \mathbf{w}_{init}^{(t)(1)} = \mathbf{w}_{init}^{(t)} - \alpha^{(t)} \left(\nabla_{\mathbf{w}}\mathcal{L}_w^{(t)}\right)\bigg|{\mathbf{w}^{init}}
\end{equation}
For notational simplicity let us first denote $\mathbf{w}_{init}^{(t)(1)} = \mathbf{w}^{(1)}$, $\mathbf{w}_{init}^{(t)} = \mathbf{w}^{(0)}$ and $\left(\nabla_{\mathbf{w}}\mathcal{L}_w^{(t)}\right)\bigg|{\mathbf{w}^{init}} = \nabla_{\mathbf{w}^{(0)}}\mathcal{L}_w^{(t)}$. We also denote the hessian of $\mathcal{L}_w^{(t)}$ at $\mathbf{w}^{(0)}$ as $\mathcal{H}_{\mathbf{w}^{(0)}}$.
Now, we can write the taylor series expansion of $\mathcal{L}_w^{(t)}$ at $\mathbf{w}^{(1)}$ as follows:
\begin{equation}
\begin{split}
  \mathcal{L}_w^{(t)}(\mathbf{w}^{(1)}) & =  \mathcal{L}_w^{(t)}(\mathbf{w}^{(0)}-\alpha^{(t)} \nabla_{\mathbf{w}^{(0)}}\mathcal{L}_w^{(t)}) \\
  & = \mathcal{L}_w^{(t)}(\mathbf{w}^{(0)}) - \alpha^{(t)} \left(\nabla_{\mathbf{w}^{(0)}}\mathcal{L}_w^{(t)}\right)^\top \left(\nabla_{\mathbf{w}^{(0)}}\mathcal{L}_w^{(t)}\right) + \frac{\left(\alpha^{(t)}\right)^2}{2} \left(\nabla_{\mathbf{w}^{(0)}}\mathcal{L}_w^{(t)}\right)^\top \mathcal{H}_{\mathbf{w}^{(0)}} \left(\nabla_{\mathbf{w}^{(0)}}\mathcal{L}_w^{(t)}\right) + \mathcal{O}((\alpha^{(t)})^3) \\
  & \approx \mathcal{L}_w^{(t)}(\mathbf{w}^{(0)}) - \alpha^{(t)} \left(\nabla_{\mathbf{w}^{(0)}}\mathcal{L}_w^{(t)}\right)^\top \left(\nabla_{\mathbf{w}^{(0)}}\mathcal{L}_w^{(t)}\right) + \frac{\left(\alpha^{(t)}\right)^2}{2} \left(\nabla_{\mathbf{w}^{(0)}}\mathcal{L}_w^{(t)}\right)^\top \mathcal{H}_{\mathbf{w}^{(0)}} \left(\nabla_{\mathbf{w}^{(0)}}\mathcal{L}_w^{(t)}\right)
\end{split}
\label{taylor}
\end{equation}

In order to minimize $\mathcal{L}_w^{(t)}(\mathbf{w}^{(1)})$ we differentiate Eqn.~\eqref{taylor} with respect to $\alpha^{(t)}$ and set it zero to get $\alpha_{best}^{(t)}$. Specifically,
\begin{equation}
\begin{split}
   & \frac{\partial \mathcal{L}_w^{(t)}(\mathbf{w}^{(1)}) }{\partial \alpha^{(t)}} \bigg|_{\alpha^{(t)} = \alpha_{best}^{(t)}}  = 0 \\
   \implies & -\left(\nabla_{\mathbf{w}^{(0)}}\mathcal{L}_w^{(t)}\right)^\top \left(\nabla_{\mathbf{w}^{(0)}}\mathcal{L}_w^{(t)}\right) + \alpha_{best}^{(t)} \left(\nabla_{\mathbf{w}^{(0)}}\mathcal{L}_w^{(t)}\right)^\top \mathcal{H}_{\mathbf{w}^{(0)}} \left(\nabla_{\mathbf{w}^{(0)}}\mathcal{L}_w^{(t)}\right) = 0 \\
   \implies & \alpha_{best}^{(t)} = \frac{\left(\nabla_{\mathbf{w}^{(0)}}\mathcal{L}_w^{(t)}\right)^\top \left(\nabla_{\mathbf{w}^{(0)}}\mathcal{L}_w^{(t)}\right)}{\left(\nabla_{\mathbf{w}^{(0)}}\mathcal{L}_w^{(t)}\right)^\top \mathcal{H}_{\mathbf{w}^{(0)}} \left(\nabla_{\mathbf{w}^{(0)}}\mathcal{L}_w^{(t)}\right)} = \frac{\left(\nabla_{\mathbf{w}}\mathcal{L}_w^{(t)}\right)^\top \left(\nabla_{\mathbf{w}}\mathcal{L}_w^{(t)}\right)}{\left(\nabla_{\mathbf{w}}\mathcal{L}_w^{(t)}\right)^\top \mathcal{H}_w \left(\nabla_{\mathbf{w}} \mathcal{L}_w^{(t)}\right)} \Bigg|_{\mathbf{w}^{init}}
\end{split}
\end{equation}

This is the desired expression of $\alpha_{best}^{(t)}$ in Eqn. 10 in the main paper. \\

Note that $\mathbf{w}^{(1)}$ does not lie within the simplex. To ensure that the updated $\mathbf{w}$ remains within the simplex, we project it onto the simplex after each gradient step. This can be done by applying the softmax operator ($\delta(.)$ in the main paper), which will ensure that the updated weights are normalized and satisfy the constraints of the simplex. Moreover, in an ideal scenario, one would calculate the optimal step size $\alpha_{best}^{(t)}$ after each gradient step, taking into account the updated point. However, for the purpose of our experiment, we calculate $\alpha_{best}^{(t)}$ only for the first step and use this value as the learning rate for the remaining steps in order to avoid hessian calculation repeatedly. In our experiment, we limit the number of steps to 5 in order to ensure quicker inference. Empirically, we have observed that using the obtained step size as fixed throughout the optimization process works reasonably well.

\vskip\medskipamount 
\leaders\vrule width \textwidth\vskip0.4pt 
\vskip\medskipamount 
\nointerlineskip

\end{document}

%% file: tables/problem_setting.tex
\begin{table}[h!]
\centering
\vspace{-1.5em}
\caption{\textbf{Comparison of our setting to the existing adaptation settings}. Our proposed setting meets all the criteria that are expected in a comprehensive adaptation framework.}
\resizebox{0.6\columnwidth}{!}{%
\begin{tabular}{ccccc}
\toprule
\textbf{Setting} & \begin{tabular}[c]{@{}c@{}}\textbf{Source}\\ \textbf{Free}\end{tabular} & \begin{tabular}[c]{@{}c@{}}\textbf{Adaptation}\\ \textbf{On the Fly}\end{tabular} & \begin{tabular}[c]{@{}c@{}}\textbf{Dynamic}\\ \textbf{Target}\end{tabular} & \begin{tabular}[c]{@{}c@{}}\textbf{Multi}\\ \textbf{Source}\end{tabular} \\ \midrule
UDA \cite{tzeng2017adversarial}
& \color{red}{\xmark}
& \color{red}{\xmark}
& \color{red}{\xmark}
& \color{ForestGreen}\cmark \\
Source-free UDA \cite{ahmed2021unsupervised}
& \color{ForestGreen}\cmark
& \color{red}{\xmark}
& \color{red}{\xmark}
& \color{ForestGreen}\cmark \\
TTA \cite{wang2020tent}
& \color{ForestGreen}\cmark
& \color{ForestGreen}\cmark
& \color{ForestGreen}\cmark
& \color{red}\xmark \\
\shortname
& \color{ForestGreen}\cmark
& \color{ForestGreen}\cmark
& \color{ForestGreen}\cmark
& \color{ForestGreen}\cmark \\ \bottomrule
\end{tabular}
}
\label{tab:prob_setting}
\vspace{-1em}
\end{table}

%% file: tables/cifar100_new.tex
\begin{table*}[h!]
\caption{\textbf{Results on CIFAR-100C.} 
We take four source models trained on \textit{Clear}, \textit{Snow}, \textit{Fog}, and \textit{Frost}. We employ these models for adaptation on 15 sequential test domains. This table illustrates that even in the dynamic environment X+ \shortname{} performs better than X-Best, which is the direct consequence of optimal aggregation of source models as well as better preservation of source knowledge. (Results in error rate $\downarrow$ (in \%))}
\centering
\footnotesize
\resizebox{1\columnwidth}{!}{
\begin{tabular}{cccccccccccccccc|c}
\hline
             & GN   & SN   & IN   & DB   & GB   & MB   & ZB   & Snow & Frost & Fog  & Bright & Contrast & Elastic & Pixel & JPEG & Mean \\ \hline
Source Worst & 97.7 & 96.5 & 98.2 & 68.8 & 78.1 & 66.1 & 65.1 & 53.6 & 59.3  & 62.0 & 55.8   & 95.4     & 61.9    & 71.5  & 75.2 & 73.7 \\
Source Best  & 90.5 & 89.0 & 94.5 & 50.7 & 48.1 & 51.9 & 44.5 & 30.0 & 29.5  & 28.2 & 39.0   & 81.9     & 44.0    & 38.5  & 57.1 & 54.5 \\ \hline
Tent Worst   & 55.9 & 55.6 & 71.2 & 58.0 & 75.5 & 78.2 & 83.3 & 89.2 & 92.4  & 93.7 & 95.4   & 96.7     & 96.5    & 96.6  & 96.7 & 82.3 \\
Tent Best    & 45.6 & 43.8 & 59.1 & 48.5 & 59.1 & 59.1 & 60.4 & 65.6 & 66.1  & 76.7 & 75.3   & 89.8     & 89.0    & 91.3  & 94.2 & 68.2 \\
\rowcolor{Gray}
Tent + \shortname  & \bf{42.2} & \bf{40.6} & \bf{55.3} & \bf{28.6} & \bf{40.7} & \bf{31.9} & \bf{29.6} & \bf{31.7} & \bf{32.4}  & \bf{30.9} & \bf{28.6}   & \bf{41.5}     & \bf{38.5}    & \bf{34.8}  & \bf{49.9} & \bf{37.1} \\ \hline
EaTA Worst   & 57.7 & 54.0 & 66.5 & 40.6 & 53.2 & 41.4 & 36.8 & 44.0 & 43.5  & 45.4 & 34.8   & 45.4     & 45.7    & 39.9  & 55.7 & 47.0 \\
EaTA Best    & 48.1 & 44.7 & 57.9 & 37.1 & 44.1 & 38.7 & 34.9 & 33.7 & 31.9  & 31.6 & 33.2   & 37.2     & 40.0    & 34.7  & 50.3 & 39.9 \\
\rowcolor{Gray}
EaTA + \shortname  & \bf{43.3} & \bf{40.7} & \bf{54.3} & \bf{27.5} & \bf{39.4} & \bf{30.4} & \bf{27.5} & \bf{29.2} & \bf{29.1}  & \bf{28.3} & \bf{25.9}   & \bf{31.3}     & \bf{33.4}    & \bf{29.0}  & \bf{43.1} & \bf{34.2} \\ \hline
CoTTA Worst  & 59.2 & 57.4 & 68.0 & 40.1 & 52.7 & 42.1 & 40.5 & 47.0 & 46.6  & 47.2 & 39.4   & 43.6     & 44.5    & 41.4  & 47.4 & 47.8 \\
CoTTA Best   & 49.8 & 46.6 & 58.6 & 34.0 & 40.7 & 36.5 & 34.2 & 34.2 & 32.8  & 33.0 & 32.8   & 34.8     & 35.3    & 33.6  & 41.1 & 38.5 \\
\rowcolor{Gray}
CoTTA + \shortname & \bf{44.6} & \bf{43.8} & \bf{57.2} & \bf{27.8} & \bf{37.6} & \bf{30.6} & \bf{28.0} & \bf{29.3} & \bf{29.3} & \bf{28.2} & \bf{26.6} & \bf{30.0} & \bf{32.5} & \bf{29.7} & \bf{41.4} & \bf{34.4} \\ \hline
\end{tabular}}
\label{tab:cifar100}
\vspace{-1em}
\end{table*}

%% file: tables/office_new.tex

\begin{wraptable}{r}{0.6\textwidth} 
    \caption{\textbf{Results on Office-Home.} We train three source models using three domains in this dataset and use them for testing on the remaining domain under the TTA setting. Our results demonstrate that X+CONTRAST consistently outperforms all of the baselines (X) ($\%$ error).}
    \centering
    \footnotesize
    \resizebox{0.5\columnwidth}{!}{
    \begin{tabular}{ccccc|c}
    \hline
                 & Ar   & Cl   & Pr   & Rw   & Avg. \\ \hline
    Source Worst & 61.4 & 64.9 & 46.2 & 43.9 & 54.1 \\
    Source Best  & 42.5 & 58.5 & 29.8 & 35.7 & 41.6 \\ \hline
    Tent Worst   & 57.7 & 60.4 & 46.5 & 42.1 & 51.7 \\
    Tent Best    & 41.4 & 54.3 & 27.9 & 36.0 & 39.9 \\
    \rowcolor{Gray}
    Tent + CONTRAST   & \bf{40.7} & \bf{52.5} & \bf{27.4} & \bf{27.4} & \bf{37.0} \\ \hline
    EaTA Worst   & 58.4 & 64.3 & 48.0 & 43.5 & 53.5 \\
    EaTA Best    & 42.1 & 57.8 & 30.3 & 35.9 & 41.5 \\
    \rowcolor{Gray}
    EaTA + CONTRAST   & \bf{40.1} & \bf{53.3} & \bf{28.3} & \bf{28.0} & \bf{37.4} \\ \hline
    CoTTA Worst  & 58.3 & 62.9 & 47.1 & 42.8 & 52.8 \\
    CoTTA Best   & 42.1 & 55.0 & 29.0 & 34.9 & 40.2 \\
    \rowcolor{Gray}
    CoTTA + CONTRAST  & \bf{40.6} & \bf{53.3} & \bf{28.3} & \bf{29.0} & \bf{37.8} \\ \hline
    \end{tabular}}
    \label{tab:office}
\end{wraptable}

%% file: tables/digit_new.tex
\begin{table}[h!]
\caption{\textbf{Results on Digits dataset.} We train the source models using four digit datasets to perform inference on the remaining dataset. The column abbreviations correspond to the datasets as follows: `MM' for MNIST-M, `MT' for MNIST, `UP' for USPS, `SV' for SVHN, and `SY' for Synthetic Digits.. The table (reporting $\%$ error rate($\downarrow$)) shows that X+\shortname{} outperforms all of the baselines (X-Best) consistently .}
\centering
\footnotesize
\resizebox{0.6\columnwidth}{!}{
\begin{tabular}{cccccc|c}
\hline
             & MM   & MT   & UP   & SV   & SY   & Avg. \\ \hline
Source Worst & 80.5 & 59.4 & 50.3 & 88.5 & 84.8 & 72.7 \\
Source Best  & 47.7 & 2.2  & 16.8 & 18.3 & 6.7  & 18.3 \\ \hline
Tent Worst   & 84.2 & 46.9 & 41.1 & 90.1 & 85.4 & 69.5 \\
Tent Best    & 45.2 & 2.3  & 16.7 & 14.4 & 6.7  & 17.1 \\
\rowcolor{Gray}
Tent + \shortname  & \bf{37.5} & \bf{1.9}  & \bf{11.2} & \bf{14.2} & \bf{6.7}  & \bf{14.3} \\ \hline
EaTA Worst   & 80.1 & 48.4 & 42.6 & 88.0 & 83.1 & 68.4 \\
EaTA Best    & 47.1 & 2.7  & 18.2 & 18.5 & 7.2  & 18.7 \\
\rowcolor{Gray}
EaTA + \shortname  & \bf{39.5} & \bf{2.0}  & \bf{11.5} & \bf{18.0} & \bf{7.0}  & \bf{15.6} \\ \hline
CoTTA Worst  & 80.0 & 48.3 & 42.8 & 87.9 & 82.9 & 68.4 \\
CoTTA Best   & 47.0 & 2.8  & 18.6 & 18.5 & 7.2  & 18.8 \\
\rowcolor{Gray}
CoTTA + \shortname & \bf{39.6} & \bf{2.0}  & \bf{11.7} & \bf{18.1} & \bf{7.1}  & \bf{15.7} \\ \hline
\end{tabular}}
\label{tab:digit}
\vspace{-1em}
\end{table}

%% file: tables/cifar10_new.tex
\begin{table*}[h!]
\caption{\textbf{Results on CIFAR-10C.} 
We take four source models trained on \textit{Clear}, \textit{Snow}, \textit{Fog} and \textit{Frost}. We employ these models for adaptation on 15 sequential test domains. This table illustrates that even in the dynamic environment X+\shortname~ performs better than X, which is the direct consequence of better retaining source knowledge. (Results in error rate $\downarrow$ (in \%))}
\centering
\footnotesize
\resizebox{\columnwidth}{!}{
\begin{tabular}{cccccccccccccccc|c}
\hline
             & GN   & SN   & IN   & DB   & GB   & MB   & ZB   & Snow & Frost & Fog  & Bright & Contrast & Elastic & Pixel & JPEG & Mean \\ \hline
Source Worst & 84.7 & 81.1 & 89.1 & 42.6 & 55.6 & 36.2 & 32.2 & 30.6 & 39.2  & 28.7 & 18.5   & 76.4     & 26.9    & 50.0  & 32.7 & 48.3 \\
Source Best  & 72.1 & 67.8 & 76.5 & 22.8 & 20.4 & 26.6 & 18.7 & 8.1  & 8.2   & 6.9  & 10.6   & 56.8     & 18.8    & 13.9  & 23.9 & 30.1 \\ \hline
Tent Worst   & 26.6 & 22.7 & 36.1 & 20.0 & 34.9 & 28.8 & 28.7 & 32.8 & 34.4  & 36.1 & 30.3   & 38.2     & 44.8    & 41.7  & 46.8 & 33.5 \\
Tent Best    & 19.3 & 17.6 & 27.9 & 14.5 & 21.1 & 17.6 & 13.5 & 14.3 & 12.6  & 14.4 & 12.4   & 17.0     & 19.0    & 14.3  & 20.4 & 17.1 \\
\rowcolor{Gray}
Tent + CONTRAST  & \bf{17.2} & \bf{15.6} & \bf{25.7} & \bf{9.1}  & \bf{19.1} & \bf{11.7} & \bf{9.0}  & \bf{9.9}  & \bf{10.1}  & \bf{9.7}  & \bf{7.7}    & \bf{11.7}    & \bf{14.5}    & \bf{10.3}  & \bf{17.4} & \bf{13.2} \\ \hline
EaTA Worst   & 31.5 & 30.4 & 44.8 & 14.8 & 33.9 & 16.1 & 13.4 & 20.5 & 21.6  & 19.3 & 11.2   & 18.9     & 23.2    & 19.5  & 29.6 & 23.2 \\
EaTA Best    & 21.9 & 20.8 & 33.9 & 10.5 & 19.6 & 14.3 & 10.6 & 8.6  & 9.0   & \bf{7.5}  & 8.5    & 10.3     & 16.1    & 11.4  & 24.0 & 15.1 \\
\rowcolor{Gray}
EaTA + CONTRAST  & \bf{18.0} & \bf{17.3} & \bf{29.4} & \bf{8.3}  & \bf{18.2} & \bf{10.0} & \bf{7.5}  & \bf{8.0}  & \bf{8.4}   & 7.9  & \bf{6.4}    & \bf{9.1}      & \bf{13.1}    & \bf{10.0}  & \bf{18.1} & \bf{12.6} \\ \hline
CoTTA Worst  & 30.1 & 26.8 & 37.8 & 15.0 & 28.5 & 16.6 & 14.6 & 19.3 & 18.6  & 17.5 & 12.2   & 15.9     & 19.4    & 15.4  & 19.3 & 20.5 \\
CoTTA Best   & 21.0 & 18.5 & 28.0 & 11.2 & \bf{17.3} & 13.3 & 11.1 & 10.6 & 10.4  & 9.5  & 9.7    & 11.2     & 13.1    & 10.5  & 15.6 & 14.1 \\
\rowcolor{Gray}
CoTTA + CONTRAST & \bf{18.4} & \bf{17.0} & \bf{28.0} & \bf{8.4}  & 17.7 & \bf{10.7} & \bf{7.9}  & \bf{9.1}  & \bf{8.4}   & \bf{8.5}  & \bf{6.8}    & \bf{8.3}      & \bf{12.1}    & \bf{9.3}  & \bf{15.3} & \bf{12.4} \\ \hline
\end{tabular}}
\label{tab:cifar10}
\end{table*}

%% file: tables/ablation.tex
\begin{table}[h!]
\caption{\textbf{Effect of initialization and step size choice.} Error rate on Office-Home under different choices of initialization and step sizes.}
\vskip 0.08in
\centering
\resizebox{0.6\columnwidth}{!}{
\begin{tabular}{lcccccc}
\toprule
\multicolumn{1}{c}{}&\multicolumn{6}{c}{\textbf{Step size}} \\\cmidrule(lr){2-7}
 \textbf{Initialization}   & $1e-3$ & $1e-2$ & $1e-1$   & $1e0$  & $1e1$    & Ours \\ \midrule
Random & 40.7 & 40.9 & 40.6 & 39.6 & 41.5 & 39.3     \\ 
\rowcolor{Gray}
Ours   & 37.9 & 37.8 & 37.5  & 37.4  & 39.1  & \textbf{37.0}  \\ \bottomrule
\end{tabular}
}
\label{tab:ablation}
\end{table}

%% file: tables/ablation_entropy_init.tex
\begin{table*}[h!]
\caption{\textbf{Initialization based on Entropy.} The table shows the results of entropy based initialization. (Results in error-rate \% $\downarrow$)}
\centering
\footnotesize
\resizebox{\columnwidth}{!}{
\begin{tabular}{cccccccccccccccc|c}
\hline
Update Policy & GN & SN & IN & DB & GB & MB & ZB & Snow & Frost & Fog & Bright & Contrast & Elastic & Pixel & JPEG & Mean \\ \hline
Entropy\_init & 42.7 & 41.1 & 56.9 & 33.5 & 46.5 & 39.4 & 37.2 & 41.0 & 43.2 & 50.6 & 46.7 & 78.6 & 77.9 & 79.5 & 88.7 & 53.6 \\
Ours & 42.2 & 40.6 & 55.3 & 28.6 & 40.7 & 31.9 & 29.6 & 31.7 & 32.4 & 30.9 & 28.6 & 41.5 & 38.5 & 34.8 & 49.9 & 37.1 \\ \hline
\end{tabular}
}
\label{tab:ablation_entropy_init}
\end{table*}

%% file: tables/ablation_cotta.tex
\begin{table*}[t]
\caption{\textbf{Choice of model update (MeTA+CoTTA).} In our experiments using CoTTA as the model update method on CIFAR100-C, we tested four scenarios: updating all models, updating only the least correlated model, updating subset of model, and updating only the most correlated model. Our results indicate that our model selection approach produces the most favorable outcome. (Results in error rate $\downarrow$ (in \%))}
\centering
\footnotesize
\resizebox{\columnwidth}{!}{
\begin{tabular}{cccccccccccccccc|c}
\hline
Update Policy   & GN   & SN   & IN   & DB   & GB   & MB   & ZB   & Snow & Frost & Fog  & Bright & Contrast & Elastic & Pixel & JPEG & Mean  \\ \hline
All Model Update       & \textbf{44.0} & \textbf{42.5} & \textbf{54.5} & 30.1 & 38.9 & 33.4 & 31.7 & 32.7 & 32.1  & 32.6 & 30.2   & 32.8     & 34.5    & 32.0  & \textbf{40.2} & 36.2 \\
Least Corr. Update      & 44.8 & 44.5 & 58.9 & 28.6 & 38.7 & 31.0 & 28.4 & 29.1 & \textbf{28.9}  & 29.5 & 26.9   & 30.9     & 33.8    & 30.5  & 44.0 & 35.2 \\
Subset of Models Update & 44.5 & 43.3 & 57.1 & 28.1 & \textbf{37.5} & 30.6 & 28.4 & 29.9 & 29.9  & 28.8 & 26.8   & 30.2     & \textbf{32.4}    & 30.2  & 40.4 & 34.5 \\
\rowcolor{Gray}
Most Corr. Update       & 44.6 & 43.8 & 57.2 & \textbf{27.8} & 37.6 & \textbf{30.6} & \textbf{28.0} & \textbf{29.3} & 29.3  & \textbf{28.2} & \textbf{26.6}   & \textbf{30.0}     & 32.5    & \textbf{29.7}  & 41.4 & \textbf{34.4} \\ \hline
\end{tabular}}
\label{tab:cotta_ablation}
\end{table*}

%% file: tables/ablation_model_update.tex
\begin{table*}[h!]
\caption{\textbf{Choice of model update (CONTRAST+Tent).} In our experiments using Tent as the model update method on CIFAR100-C, we tested four scenarios: updating all models, updating only the least correlated model, updating subset of model, and updating only the most correlated model. Our results indicate that our model selection approach produces the most favorable outcome. (Results in error rate $\downarrow$ (in \%))}
\centering
\footnotesize
\resizebox{\columnwidth}{!}{
\begin{tabular}{cccccccccccccccc|c}
\hline
Update Policy      & GN   & SN   & IN   & DB   & GB   & MB   & ZB   & Snow & Frost & Fog  & Bright & Contrast & Elastic & Pixel & JPEG & Mean \\ \hline
All Model Update   & \bf{41.6} & 40.9 & 57.8 & 47.1 & 60.2 & 60.3 & 62.1 & 68.6 & 73.2  & 80.9 & 82.1   & 92.4     & 91.2    & 92.5  & 94.9 & 69.7 \\
Least Corr. Update & 43.8 & 41.4 & 56.1 & 31.2 & 41.4 & 34.8 & 31.4 & 33.5 & 33.1  & 37.5 & 31.5   & 41.6     & 41.5    & 37.5  & 53.1 & 39.3 \\
Subset of Models Update  & 43.0 & 41.1 & 56.4 & 33.0 & 47.8 & 38.7 & 37.5 & 41.4 & 45.3 & 51.1 & 46.4 & 83.6 & 81.0 & 60.1 & 92.4 & 53.3 \\
\rowcolor{Gray}
Most Corr. Update  & 42.2 & \bf{40.6} & \bf{55.3} & \bf{28.6} & \bf{40.7} & \bf{31.9} & \bf{29.6} & \bf{31.7} & \bf{32.4}  & \bf{30.9} & \bf{28.6}   & \bf{41.5}     & \bf{38.5}    & \bf{34.8}  & \bf{49.9} & \bf{37.1} \\ \hline
\end{tabular}
}
\label{tab:update}
\end{table*}

%% file: tables/ablation_weighted_model_update.tex
\begin{table*}[h!]
\caption{\textbf{Model Update according to Weight.} The table shows results of updating model according to their respective weights. (Results in error-rate \% $\downarrow$)}
\centering
\footnotesize
\resizebox{\columnwidth}{!}{
\begin{tabular}{cccccccccccccccc|c}
\hline
Update Policy & GN & SN & IN & DB & GB & MB & ZB & Snow & Frost & Fog & Bright & Contrast & Elastic & Pixel & JPEG & Mean \\ \hline
Tent & 41.7 & 39.7 & 53.0 & 33.9 & 43.9 & 36.8 & 34.6 & 37.8 & 39.3 & 41.0 & 36.8 & 56.1 & 49.5 & 41.4 & 60.1 & 43.0 \\
CONTRAST+Tent & 42.2 & 40.6 & 55.3 & 28.6 & 40.7 & 31.9 & 29.6 & 31.7 & 32.4 & 30.9 & 28.6 & 41.5 & 38.5 & 34.8 & 49.9 & 37.1 \\
CoTTA & 44.5 & 43.0 & 56.2 & 28.1 & 38.1 & 30.8 & 28.6 & 29.9 & 29.6 & 28.7 & 27.0 & 29.5 & 31.8 & 29.0 & 38.6 & 34.2 \\
CONTRAST+CoTTA & 44.6 & 43.8 & 57.2 & 27.8 & 37.6 & 30.6 & 28.0 & 29.3 & 29.3 & 28.2 & 26.6 & 30.0 & 32.5 & 29.7 & 41.4 & 34.4 \\ \hline
\end{tabular}
}
\label{tab:ablation_weighted_update}
\end{table*}

%% file: tables/ablation_msda.tex
\begin{table*}[h!]
\caption{\textbf{Comparison with MSDA.} The table compares the performance of our method with MSDA approach DECISION. (Results in error-rate \% $\downarrow$)}
\centering
\footnotesize
\resizebox{\columnwidth}{!}{
\begin{tabular}{cccccccccccccccc|c}
\hline
Update Policy & GN & SN & IN & DB & GB & MB & ZB & Snow & Frost & Fog & Bright & Contrast & Elastic & Pixel & JPEG & Mean \\ \hline
DECISION & 55.0 & 76.2 & 90.5 & 95.2 & 97.3 & 97.9 & 98.2 & 98.0 & 98.3 & 98.4 & 98.4 & 98.7 & 99.0 & 98.9 & 98.9 & 93.3 \\
CONTRAST+Tent & 42.2 & 40.6 & 55.3 & 28.6 & 40.7 & 31.9 & 29.6 & 31.7 & 32.4 & 30.9 & 28.6 & 41.5 & 38.5 & 34.8 & 49.9 & 37.1 \\ \hline
\end{tabular}
}
\label{tab:ablation_msda}
\end{table*}

%% file: tables/ablation_model_soup.tex
\begin{table*}[h!]
\caption{\textbf{Comparison with Model Soups.} The table compares the performance our method against model soups. (Results in error-rate \% $\downarrow$)}
\centering
\footnotesize
\resizebox{\columnwidth}{!}{
\begin{tabular}{cccccccccccccccc|c}
\hline
Update Policy & GN & SN & IN & DB & GB & MB & ZB & Snow & Frost & Fog & Bright & Contrast & Elastic & Pixel & JPEG & Mean \\ \hline
Model-Soups & 96.82 & 96.26 & 97.08 & 95.17 & 95.33 & 95.30 & 95.22 & 95.17 & 95.86 & 95.28 & 94.96 & 97.41 & 95.04 & 95.05 & 95.86 & 95.72 \\
CONTRAST+Tent & 42.2 & 40.6 & 55.3 & 28.6 & 40.7 & 31.9 & 29.6 & 31.7 & 32.4 & 30.9 & 28.6 & 41.5 & 38.5 & 34.8 & 49.9 & 37.1 \\ \hline
\end{tabular}
}
\label{tab:ablation_model_soup}
\end{table*}

%% file: tables/seg_static.tex
\begin{table*}[h!]
\centering
\caption{\textbf{Result on Cityscape to ACDC:} In this experiment, we test our method on the test data from individual weather conditions (static test distribution) of ACDC. The source models are trained on the train set of Cityscape and its noisy variants. Our method clearly outperforms baseline adaptation method. (Results in \% mIoU)} 
\vskip 0.08in
\begin{tabular}{cccccc}
\hline
\textbf{Method}     & \textbf{Fog}           & \textbf{Rain}          & \textbf{Snow}          & \textbf{Night}       & \textbf{Avg.}       \\ \hline
 Tent-Best  & 25.3          & 21.0            & 19.2          & 12.6        & 19.5          \\

                         CONTRAST       & \textbf{27.7} & \textbf{22.8} & \textbf{21.1} & \textbf{14.0} & \textbf{21.4} \\ \hline
\end{tabular}
\label{tab:seg_static}
\end{table*}

%% file: tables/acdc_dynamic_5steps.tex
\begin{table}[h!]
\centering
\caption{\textbf{Result on Cityscapes to ACDC for dynamic test distribution:} This table illustrates that over a prolonged cycle of repetitive
test distributions, our model can retain performance better than baseline Tent. ((Results in \% mIoU))}
\vskip 0.08in
\small
\begin{tabularx}{\columnwidth}{l|XXXXXXXXXXXXX}
\hline
Time                 & \multicolumn{1}{l}{t} & \multicolumn{1}{l}{} & \multicolumn{1}{l}{} & \multicolumn{1}{l}{}               & \multicolumn{1}{l}{}  & \multicolumn{1}{l}{} & \multicolumn{1}{l}{} & \multicolumn{1}{l}{}               & \multicolumn{1}{l}{}  & \multicolumn{1}{l}{} & \multicolumn{1}{l}{} & \multicolumn{1}{l}{}               & \multicolumn{1}{l}{}    \\ \hline
Round                & \multicolumn{1}{l}{1} & \multicolumn{1}{l}{}       & \multicolumn{1}{l}{} & \multicolumn{1}{l|}{}              & \multicolumn{1}{l}{3} & \multicolumn{1}{l}{} & \multicolumn{1}{l}{} & \multicolumn{1}{l|}{}              & \multicolumn{1}{l}{5} & \multicolumn{1}{l}{} & \multicolumn{1}{l}{} & \multicolumn{1}{l|}{}              & \multicolumn{1}{l}{All} \\ \hline
Conditions           & Rain                  & Snow                       & Fog                  & \multicolumn{1}{c|}{Night}         & Rain                  & Snow                 & Fog                  & \multicolumn{1}{c|}{Night}         & Rain                  & Snow                 & Fog                  & \multicolumn{1}{c|}{Night}         & Mean                    \\ \hline
Tent-Best            & 20.1                  & 21.3                       & 22.3                 & \multicolumn{1}{c|}{11.3}          & 18.5                  & 17.2                 & 19.5                 & \multicolumn{1}{c|}{8.4}           & 15.8                  & 14.5                 & 17.5                 & \multicolumn{1}{c|}{6.8}           & 16.1                    \\
CONTRAST             & 22.1                  & \textbf{21.4}              & \textbf{24.3}        & \multicolumn{1}{c|}{\textbf{13.4}} & \textbf{21.4}         & \textbf{18.3}        & \textbf{23.5}        & \multicolumn{1}{c|}{\textbf{11.3}} & \textbf{18.6}         & \textbf{15.5}        & \textbf{21.4}        & \multicolumn{1}{c|}{\textbf{10.4}} & \textbf{18.6}           \\ \hline
\end{tabularx}
\label{tab:seg_dynamic}
\end{table}